\pdfoutput=1
\documentclass[11pt]{article}

\usepackage[compress,numbers]{natbib}
\usepackage[left=1.25in, right=1.25in, top=1.25in, bottom=1.25in]{geometry}
\usepackage{amsmath}
\usepackage[scaled=0.92]{PTSans}
\usepackage{xfrac}

\usepackage[utf8]{inputenc} 
\usepackage[T1]{fontenc}    
\usepackage[colorlinks]{hyperref}
\hypersetup{citecolor=MidnightBlue}
\hypersetup{linkcolor=black}
\hypersetup{urlcolor=MidnightBlue}
\usepackage{url}            
\usepackage{booktabs}       
\usepackage{nicefrac}       
\usepackage{microtype}      
\usepackage[dvipsnames]{xcolor}         
\usepackage{colortbl}
\usepackage[pdftex]{contour}

\usepackage{graphicx}
\usepackage[labelfont=bf]{caption}
\usepackage{subcaption}
\usepackage{lipsum}
\usepackage{bbm}
\usepackage{cancel}
\usepackage{wrapfig}
\usepackage{adjustbox}
\usepackage{multirow}
\usepackage{geometry}
\usepackage{pdflscape}
\usepackage{changepage}

\usepackage{amsfonts}
\usepackage{amsmath}
\usepackage{amssymb}
\usepackage{amsthm}

\usepackage{algorithm}
\usepackage{algorithmicx,tikz}
\usepackage{algcompatible}
\usepackage[noend]{algpseudocode}
\usepackage{setspace}
\usepackage{bm}

\algnewcommand{\LineComment}[1]{\State \(\#\) #1}

\renewcommand{\d}{\mathrm{\ d}}

\newcommand{\bx}{\mathbf{x}}
\newcommand{\by}{\mathbf{y}}

\newcommand{\cx}{\mathcal{X}}
\newcommand{\cy}{\mathcal{Y}}

\newcommand{\cn}{\mathcal{N}}
\newcommand{\BX}{\mathbf{X}}

\newcommand{\BA}{\mathbf{A}}

\DeclareMathOperator*{\argmax}{arg\,max}

\newcommand{\btheta}{\boldsymbol\theta}
\newcommand{\bphi}{\boldsymbol\phi}
\newcommand{\bpsi}{\boldsymbol\psi}
\newcommand{\bmu}{\boldsymbol\mu}
\newcommand{\BTheta}{\boldsymbol\Theta}

\newcommand{\BPsi}{\boldsymbol\Psi}

\newcommand{\E}{\mathbb{E}}

\theoremstyle{plain}
\newtheorem{thm}{Theorem}

\newtheorem{prop}[thm]{Proposition}

\newtheorem*{thm*}{Theorem}
\newtheorem*{prop*}{Proposition}

\DeclareRobustCommand{\rchi}{{\mathpalette\irchi\relax}}
\newcommand{\irchi}[2]{\raisebox{1.1\depth}{$#1\chi$}} 

\contourlength{0.001em}
\newcommand{\sixo}{SI${\rchi}$\kern-0.25ptO}
\newcommand{\bsixo}{SI$\bm{\rchi}$\kern-0.5ptO}

\usepackage{todonotes}

\title{\sixo: Smoothing Inference with Twisted Objectives}

\author{
  Dieterich Lawson\thanks{Equal contribution.},
  \, Allan Ravent\'{o}s\footnotemark[1],
  \, Andrew Warrington\footnotemark[1],
  \, Scott Linderman \vspace{.5em} \\
  \texttt{\{jdlawson, aravento, awarring, scott.linderman\}@stanford.edu} \vspace{.5em} \\
  Stanford University
}

\begin{document}

\maketitle

\begin{abstract}
Sequential Monte Carlo (SMC) is an inference algorithm for state space models that approximates the posterior by sampling from a sequence of target distributions. The target distributions are often chosen to be the \textit{filtering} distributions, but these ignore information from future observations, leading to practical and theoretical limitations in inference and model learning.
We introduce \emph{SIXO}, a method that instead learns targets that approximate the \emph{smoothing} distributions, incorporating information from all observations. The key idea is to use density ratio estimation to fit functions that warp the filtering distributions into the smoothing distributions. We then use SMC with these learned targets to define a variational objective for model and proposal learning. SIXO yields provably tighter log marginal lower bounds and offers significantly more accurate posterior inferences and parameter estimates in a variety of domains.
\end{abstract}

\section{Introduction}
\label{sec:intro}

In this work we consider model learning and approximate posterior inference in probabilistic state space models.
Sequential Monte Carlo (SMC) is a general-purpose method for these problems~\citep{doucet2009tutorial, naesseth2019elements,del2004feynman} that produces an unbiased estimate of the marginal likelihood as well as latent state trajectories (i.e. \textit{particles}) that can be used to approximate posterior expectations.
SMC can facilitate model learning via expectation-maximization or direct maximization of the marginal likelihood estimate~\citep{andrieu2004particle,gu2015neural}.
It can also be cast in a variational framework~\cite{blei2017variational,wainwright2008graphical} as a rich family of approximate posterior distributions that can be fit using stochastic gradient ascent and modern automatic differentiation methods~\citep{maddison2017filtering, naesseth2018variational, le2017auto, lawson2018twisted, lawson2019energy}.

The quality of SMC's marginal likelihood and posterior estimates is driven by two design decisions: the choice of \textit{proposal distributions} and \textit{target distributions}. The proposal distributions specify how particles propagate from one time step to the next while the target distributions specify how those particles are weighted and which ones survive to future time steps. The most common SMC variant, filtering SMC, sets the targets to the \textit{filtering distributions}, the conditional distributions over latent states~$\bx_{1:t} = (\bx_1, \ldots, \bx_t)$ given observations~$\by_{1:t}=(\by_1 , \ldots , \by_t)$. The central issue is that the filtering distributions do not incorporate information from future observations $\by_{t+1:T}$.

\begin{figure}[t]
    \centering
    \vspace{-1mm}
    \begin{subfigure}{\textwidth} 
            \refstepcounter{subfigure}\label{fig:banner:dists:filtering}
            \refstepcounter{subfigure}\label{fig:banner:dists:smoothing}
            \refstepcounter{subfigure}\label{fig:banner:lineages:fivo}
            \refstepcounter{subfigure}\label{fig:banner:lineages:sixo}
    \end{subfigure}%

    \includegraphics[width=\textwidth]{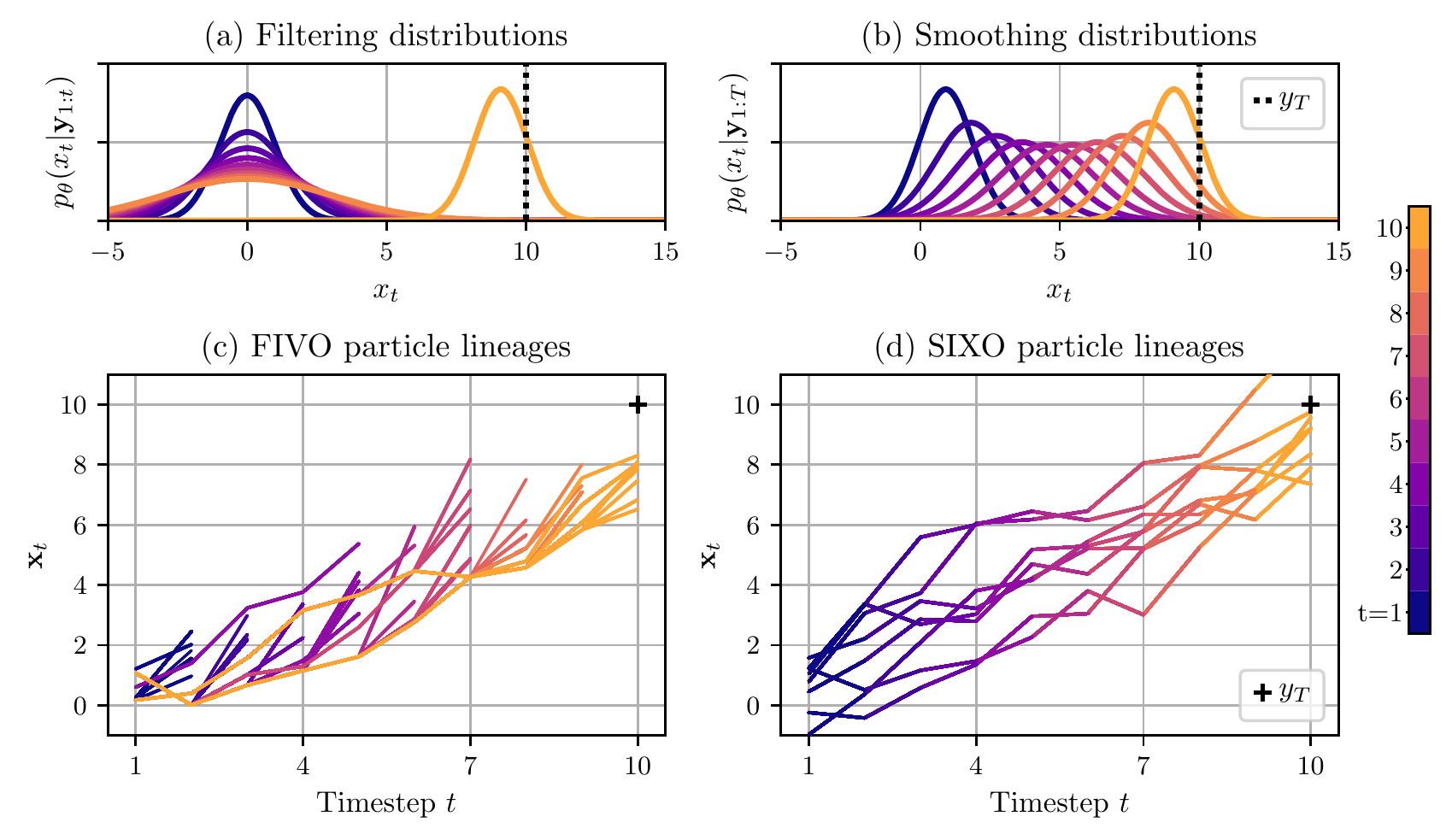}
    \caption{Theoretical and empirical target distributions for a Gaussian random walk with a single observation~$y_{T}=10$ and $\mathbf{y}_{1:T-1} = \varnothing$. \textbf{(\protect\subref{fig:banner:dists:filtering})} For $t=1,\ldots,9$ the filtering distributions reduce to a series of mean-zero Gaussians. At $t=T=10$, the filtering distribution incorporates the observation $y_T$, resulting in a sudden shift and particle death. \textbf{(\protect\subref{fig:banner:dists:smoothing})} In contrast, the smoothing distributions steadily shift towards the observation, matching the posterior perfectly.
    \textbf{(\protect\subref{fig:banner:lineages:fivo})} The proposal learned by a previous method, FIVO \cite{maddison2017filtering,naesseth2018variational,le2017auto}, exploits smoothing information to propose particles upwards towards the observed value. However, FIVO is based on filtering SMC which ``resists'' this by resampling particles back towards the prior, resulting in particle degeneracy. \textbf{(\protect\subref{fig:banner:lineages:sixo})} SIXO's proposal also leverages smoothing information, but proposed particles are preserved by the learned target distributions.}
    \label{fig:banner}
    \vspace{-.25cm}
\end{figure}

Figure~\ref{fig:banner} illustrates why setting the target distributions to the filtering distributions can be problematic.
In this example, the latent states follow a simple Gaussian random walk, but only the last step is observed.
Thus, the filtering distributions reduce to the prior, a series of mean zero Gaussians shown in Figure~\ref{fig:banner:dists:filtering}. 
If the observation is far from the prior, the filtering distribution suddenly jumps at time $T=10$.
This is a recipe for disaster in SMC: the particles at time $T-1$ will be distributed according to a mean zero Gaussian and very few will survive to the next time step, causing the variance of the SMC estimator to explode. Even if the proposals incorporate smoothing information, using filtering targets can cause particle degeneracy by resampling away high-quality particles, as seen in Figure \ref{fig:banner:lineages:fivo}.


Suppose instead that the target distributions were the \textit{smoothing distributions}---the conditional distributions over latents $\bx_{1:t}$ given \textit{all} observations $\by_{1:T}$. 
Figure~\ref{fig:banner:dists:smoothing} shows the smoothing distributions for the simple Gaussian random walk.
Unlike the filtering distributions, the smoothing distributions shift steadily toward the observation over time.
These slow, smooth changes are ideal for SMC: Figure \ref{fig:banner:lineages:sixo} shows many particles surviving from one step to the next, resulting in a low-variance SMC estimator.

In practice we do not have access to the smoothing distributions---if we did, there would be no need for SMC! Here, we introduce a new method called SIXO: Smoothing Inference with Twisted Objectives. SIXO provides a unified approach for learning model and proposal parameters, as well as a set of \textit{twisting functions} that warp the filtering distributions into targets that better approximate the smoothing distributions~\cite{whiteley2014twisted}. Like its predecessor FIVO~\citep{maddison2017filtering, naesseth2018variational, le2017auto}, SIXO uses a variational approach, deriving a lower bound to the marginal likelihood.
Unlike its predecessor, we prove that the SIXO bound can become tight, even with finitely many particles.

The key challenge with SIXO is learning the twisting functions.
We find that a simple density ratio estimation approach works best, and we propose an algorithm that interleaves twist updates with updates to the model and proposal. Thus, SIXO offers a means of jointly learning model parameters, SMC proposals, and targets for accurate posterior inference.

Finally, we give empirical evidence to support our theoretical claims. 
Across a range of experiments with a Gaussian diffusion, a stochastic volatility model of currency exchange rates, and a Hodgkin-Huxley model of membrane potential in a neuron, SIXO consistently outperforms FIVO and related methods.
We dissect these results to illustrate how learning better targets enables more effective posterior inference and model learning.
\section{Background}
\label{sec:background}

Consider modeling sequential data $\by_{1:T} \in \cy^T$ using latent variables $\bx_{1:T} \in \cx^T$ with Markovian structure, and let the joint distribution factorize as
\begin{align}
    p_{\btheta} ( \bx_{1:T}, \by_{1:T} ) = p_{\btheta} ( \bx_1 ) p_{\btheta} ( \by_1 \mid \bx_1 ) \prod_{t=2}^T p_{\btheta} ( \bx_t \mid \bx_{t-1} ) p_{\btheta} ( \by_t \mid \bx_{t} ) \label{equ:ssm}
\end{align}
with global parameters $\btheta \in \Theta$. We further assume that the conditional distributions $p_{\btheta}(\bx_t \mid \bx_{t-1})$ and $p_{\btheta}(\by_t \mid \bx_t)$ may depend nonlinearly on $\bx_{t-1}$ and $\bx_t$, respectively.

The marginal likelihood and posterior for this model class are not readily available from the joint distribution due to the intractable integral over the latents $\bx_{1:T}$, i.e.
\[
    p_{\btheta}(\by_{1:T}) = \int_{\cx^{T}}  p_{\btheta} ( \by_{1:T}, \bx_{1:T} ) \d \bx_{1:T}, 
\]
cannot easily be computed due to the form of the conditional distributions.

\subsection{Sequential Monte Carlo} 
Sequential Monte Carlo is an algorithm for inference in state-space models that approximates the posterior $p_{\btheta}(\bx_{1:T}\mid\by_{1:T})$ with a set of $K$ weighted particles $\bx_{1:T}^{1:K}$. These particles are constructed by approximately sampling from a sequence of target distributions $\{\pi_t(\bx_{1:t})\}_{t=1}^T$, with the intuition that sampling from a series of distributions that gradually approach the posterior is easier than attempting to sample it directly. The targets are often only available up to an unknown normalizing constant $Z_t$, so SMC uses the unnormalized targets $\{\gamma_t(\bx_{1:t})\}_{t=1}^T$ which correspond to the normalized targets via $\pi_t(\bx_{1:t}) = \gamma_t(\bx_{1:t})/Z_t$. 

SMC repeats three steps: First, a set of latents are sampled from a proposal distribution $q_{\btheta}(\bx^k_t \mid \bx^k_{t-1}, \by_{1:T})$ conditional on the current particles $\bx_{1:t-1}^{1:K}$. Then, each particle is weighted using the unnormalized target $\gamma_t(\bx_{1:t})$ to form an empirical approximation of the normalized target distribution. Finally, new particle trajectories $\bx_{1:t}^{1:K}$ are drawn from this approximation to the normalized target.

Ideally the target distributions smoothly approach the posterior so that sampling from the target at time $t+1$ is easy given samples from the target at time $t$. However, as long as mild technical conditions are met and ${\gamma_T(\bx_{1:T}) \propto p_{\btheta}(\bx_{1:T}, \by_{1:T})}$, SMC returns a consistent and unbiased estimate of the marginal likelihood~$p_{\btheta}(\by_{1:T})$ and a set of weighted particles approximating the posterior~$p_{\btheta}(\bx_{1:T} \mid \by_{1:T})$~\citep{doucet2009tutorial, naesseth2019elements,del2004feynman}. For more details see Appendix \ref{app:sec:background:smc} and for a thorough treatment of SMC see~\citep{doucet2009tutorial, naesseth2019elements,del2004feynman}.

\subsection{Filtering SMC and Model Learning}
\label{sec:background:filtsmc}

The most commonly-used SMC algorithm is filtering SMC, which sets the normalized targets to the \emph{filtering} distributions, i.e. $\pi_t(\bx_{1:t}) = p_{\btheta}(\bx_{1:t} \mid \by_{1:t})$ and $\gamma_t(\bx_{1:t}) \propto p_{\btheta}(\bx_{1:t}, \by_{1:t})$. Let $\widehat{Z}_\mathrm{FSMC}(\btheta, \by_{1:T})$ be the marginal likelihood estimator returned from running filtering SMC with proposal distributions $\{q_{\btheta}(\bx_t \mid \bx_{1:t-1}, \by_{1:t})\}_{t=1}^T$ which may share parameters with $p_{\btheta}$.

Previous work used filtering SMC to fit model parameters by ascending a lower bound on the log marginal likelihood called a \emph{filtering variational objective} (FIVO) \cite{maddison2017filtering,naesseth2018variational,le2017auto}. The FIVO bound is derived using Jensen's inequality and the unbiasedness of $\widehat{Z}_\mathrm{FSMC}$,
\begin{equation}
\mathcal{L}_\mathrm{FIVO}(\btheta, \by_{1:T}) \triangleq
\mathbb{E}[ \log \widehat{Z}_\mathrm{FSMC}(\btheta, \by_{1:T})] \leq
\log \mathbb{E}[\widehat{Z}_\mathrm{FSMC}(\btheta, \by_{1:T})] = 
\log p_{\btheta}(\by_{1:T}),\label{eq:methods:fivo}
\end{equation}
and is optimized using stochastic gradient ascent in $\btheta$ \cite{maddison2017filtering, naesseth2018variational, le2017auto, mnih2016variational}.

\subsection{Smoothing SMC via Twisting Functions}
\label{sec:background:smooth}
The main disadvantage of filtering SMC is that the filtering distributions only condition on observations up to the current timestep $t$, ignoring future observations $\by_{t+1:T}$. This creates situations where future observations are highly unlikely given the current latent trajectories, which in turn causes particle death, high variance in the normalizing constant estimator, and poor inference and model learning~\citep{maddison2017filtering,whiteley2014twisted,briers2010smoothing}. Performing smoothing SMC would resolve this issue by choosing the \emph{smoothing} distributions as targets, i.e. $\pi_t(\bx_{1:t}) = p_{\btheta}(\bx_{1:t} \mid \by_{1:T})$ and $\gamma_t(\bx_{1:t}) \propto p_{\btheta}(\bx_{1:t}, \by_{1:T})$. Unfortunately, $p_{\btheta}(\bx_{1:t}, \by_{1:T})$ is not readily available from the model and computing it is roughly as hard as the original inference problem.

However, $p_{\btheta}(\bx_{1:t}, \by_{1:T})$ factors into the product of the filtering distributions, $p_{\btheta}(\bx_{1:t}, \by_{1:t})$, and the \textit{lookahead} distributions, $p_{\btheta}(\by_{t+1:T} \mid \bx_t)$ (Appendix \ref{app:sec:background:twist-factor}). If the lookahead distributions can be well-approximated by a series of ``twisting'' functions \cite{whiteley2014twisted}, $\{r(\by_{t+1:T}, \bx_t)\}_{t=1}^T$, then running SMC with targets $\gamma_t(\bx_{1:t}) = p_{\btheta}(\bx_{1:t}, \by_{1:t})r(\by_{t+1:T}, \bx_t)$ would approximate smoothing SMC. In this sense, the lookahead distributions are optimal twisting functions \cite{whiteley2014twisted,guarniero2017iterated}.

Different twisting functions yield different SMC methods such as auxiliary particle filters and twisted particle filters~\citep{whiteley2014twisted, pitt1999filtering}. However, as long as the final unnormalized target $\gamma_T(\bx_{1:T})$ is proportional to $p_{\btheta}(\bx_{1:T}, \by_{1:T})$ and regularity conditions are met, SMC will produce an unbiased estimate of the marginal likelihood regardless of the choice of twisting functions~\citep{doucet2009tutorial,del2004feynman,lindsten2018graphical}. Instead, the quality of the twisting functions affects the variance of SMC's marginal likelihood estimate.
\section{SIXO: Model Learning with Smoothing SMC}
\label{sec:methods}

Our goal is to fit models by optimizing a lower bound on their log marginal likelihood constructed using smoothing SMC. To construct the lower bound, fix $r_{\bpsi}(x_T) = 1$ and let $\widehat{Z}_\mathrm{SIXO}(\btheta, \bpsi, \by_{1:T})$ be the marginal likelihood estimator returned from running SMC with unnormalized targets $\{ p_{\btheta}(\bx_{1:t}, \by_{1:t})r_{\bpsi}(\by_{t+1:T}, \bx_t)\}_{t=1}^T$ and proposal distributions $\{q_{\btheta}(\bx_t \mid \bx_{1:t-1}, \by_{1:T})\}_{t=1}^T$. Because the $T$\textsuperscript{th} unnormalized target is $p_{\btheta}(\bx_{1:T}, \by_{1:T})$, $\widehat{Z}_\mathrm{SIXO}$ will be an unbiased estimator of the marginal likelihood $p_{\btheta}(\by_{1:T})$ \citep{doucet2009tutorial,del2004feynman}. This implies via Jensen's inequality that
\begin{equation}
\begin{aligned}
    \mathcal{L}_\mathrm{SIXO}(\btheta, \bpsi, \by_{1:T}) &\triangleq \mathbb{E}\left[\log \widehat{Z}_\mathrm{SIXO}(\btheta, \bpsi, \by_{1:T})\right] \\
    &\leq \log \mathbb{E}\left[ \widehat{Z}_\mathrm{SIXO}(\btheta, \bpsi, \by_{1:T})\right]  = \log p_{\btheta}(\by_{1:T})  \label{equ:sixo_bound}
\end{aligned}
\end{equation}
i.e. $\mathcal{L}_\mathrm{SIXO}(\btheta, \bpsi, \by_{1:T})$ is a lower bound on the log marginal likelihood $\log p_{\btheta}(\by_{1:T})$ \cite{mnih2016variational}.

\subsection{The Functional Form of the Twists}
\label{sec:methods:twist-functional-form}

The structure of the lookahead distributions $p_{\btheta}(\by_{t+1:T} \mid \bx_t)$ suggests a functional form for $r_{\bpsi}$ that accepts a single latent $\bx_t$ and produces distributions over all future observations $\by_{t+1:T}$. Because the twists will be evaluated once per particle and timestep in an SMC sweep, this functional form would lead to an algorithm with $O(T^2)$ complexity. To reduce the complexity, we consider two methods: fixed-lag twisting and backwards twisting.
\\\\
\textbf{Fixed-lag twisting} approximates the full lookahead distribution $p_{\btheta}(\by_{t+1:T} \mid \bx_t)$ using a fixed window of $L$ observations, i.e. it models $p_{\btheta}(\by_{t+1:t+L} \mid \bx_t)$ \cite{pitt1999filtering,bernardo1999fixed,lin2013lookahead}. We define the fixed-lag twisting functions $\{r_{\bpsi}(\by_{t+1:t+L}, x_t)\}_{t=1}^{T-1}$ as a sequence of functions which accept $\bx_t \in \mathcal{X}$ and produce a distribution over $\by_{t+1:t+L} \in \mathcal{Y}^{L}$. This reduces the computational complexity to $O(TL)$ at the cost of only looking at $L$ observations.

In our experiments we use an $L=1$ twist that scores the next observation by approximating the one-step lookahead
\begin{align}
p_{\btheta}(\by_{t+1}\mid \bx_t) = \int p_{\btheta}(\by_{t+1} \mid \bx_{t+1}) p_{\btheta}(\bx_{t+1} \mid \bx_t)\d\bx_{t+1}
\end{align}
with Gauss-Hermite quadrature \cite{abramowitz1964handbook}. We refer to this as the ``quadrature twist''.
\\\\
\textbf{Backwards twisting} is motivated by rewriting the lookahead distributions using Bayes' rule,
\begin{align}
p_{\btheta}(\by_{t+1:T} \mid \bx_t) = \frac{p_{\btheta}(\bx_t \mid \by_{t+1:T}) \, p_{\btheta}(\by_{t+1:T})}{p_{\btheta}(\bx_t)} \propto \frac{p_{\btheta}(\bx_t \mid \by_{t+1:T})}{p_{\btheta}(\bx_t)}, \label{equ:methods:bwd_ratio}
\end{align}
dropping terms independent of $\bx_t$ because the twisting functions will be used to score particles in SMC. Thus, we need only approximate $p_{\btheta}(\bx_t \mid \by_{t+1:T})/p_{\btheta}(\bx_t)$. The numerator $p_{\btheta}(\bx_t \mid \by_{t+1:T})$ is the reverse of the lookahead distributions---it is a distribution over a single latent conditioned on future observations. This makes it possible to parameterize the twists using a recurrent function approximator (e.g. a recurrent neural network or RNN) run backwards across the observations $\by_{1:T}$ to produce twist values for each timestep. 

Specifically, we define the backwards twists $\{r_{\bpsi}(\by_{t+1:T}, \bx_t)\}_{t=1}^{T-1}$ as a sequence of positive, integrable, real-valued functions $\mathcal{Y}^{T-t} \times \mathcal{X} \to \mathbb{R}_+$ with parameters $\bpsi \in \mathbf{\Psi}$. 
Parameterizing backward twists with a recurrent function approximator results in $O(T)$ time complexity and allows the twist to condition on all future observations, making backwards twisting preferable to fixed-lag twisting.

\subsection{Learning Twists}
\label{sec:methods:learning}

\paragraph{Ascending the Unified Objective} One way to fit the twists, proposal, and model is to ascend $\mathcal{L}_\mathrm{SIXO}$ in the parameters of $p_{\btheta}, q_{\btheta}$, and $r_{\bpsi}$, similar to FIVO~\cite{maddison2017filtering, naesseth2018variational, le2017auto}. 
The gradients of this objective include score-function terms that arise from the discrete resampling steps in SMC. We refer to ascending $\mathcal{L}_\mathrm{SIXO}$ with these unbiased gradients as SIXO-u. Because the resampling gradient terms have high variance, SIXO-u is impractical for complex settings \cite{maddison2017filtering, lawson2018twisted}. For a detailed discussion and derivation of the gradient, see Appendix \ref{app:sec:methods:unified-grads}.

\paragraph{Density Ratio Estimation}
\label{sec:methods:dre}

\begin{figure}
\algrenewcommand\algorithmicindent{0.75em}
\centering
\begin{algorithm}[H]
    \caption{SIXO-DRE}
    \label{algo:sixo-dre}
    \begin{spacing}{1.1}
    \begin{algorithmic}[1]
    
    \Procedure{SIXO-DRE}{$\by_{1:T}$, $\btheta_0$, $\bpsi_0$, $S$, $N$, $K$}
        \For {$s=1,\ldots, S$}
            \State $\bpsi_s = $\Call{TWIST-UPDATE}{$\btheta_{s-1}, \bpsi_{s-1}, N$}
            \State $\btheta_s = $\Call{MODEL-UPDATE}{$\by_{1:T}, \btheta_{s-1}$, $\bpsi_{s}$, $N$, $K$}
        \EndFor
        \vspace{-1mm}
    \State \Return $\btheta_S, \bpsi_S$
    \EndProcedure
    \vspace{1mm}    
    \Procedure{TWIST-UPDATE}{$\btheta, \bpsi_0, N$}
        \For {$i=1,\ldots,N$}
            \State $\tilde{\bx}_{1:T} \sim p_{\btheta}(\bx_{1:T})$
            \State $\bx_{1:T}, \by_{1:T} \sim p_{\btheta}(\bx_{1:T}, \by_{1:T})$
            \State $\mathcal{L}_\mathrm{DRE}(\bpsi)= \frac{1}{T-1} \sum_{t=1}^{T-1} \log \sigma(\log r_{\bpsi}(\by_{t+1:T}, \bx_t)) + \log (1 - \sigma(\log r_{\bpsi}(\by_{t+1:T}, \tilde{\bx}_t)))$
            \State Compute $\bpsi_i$ using the gradients of $\mathcal{L}_\mathrm{DRE}$ evaluated at $\bpsi_{i-1}$
        \EndFor
        \vspace{-1mm}
    \State \Return $\bpsi_N$
    \EndProcedure
    \vspace{1mm}
    \Procedure{MODEL-UPDATE}{$\by_{1:T}, \btheta_0, \bpsi, N, K$}
        \For {$i=1,\ldots,N$}
            \State $\widehat{Z}_\mathrm{SIXO}(\btheta)$ = \Call{SMC}{$\{ p_{\btheta}(\bx_{1:t}, \by_{1:t})r_{\bpsi}(\by_{t+1:T}, \bx_t)\}_{t=1}^T$, $\{q_{\btheta}(\bx_t \mid \bx_{t-1}, \by_{1:T})\}_{t=1}^T$, K}
            \State Compute $\btheta_i$ using the biased gradients of $\widehat{Z}_\mathrm{SIXO}$ evaluated at $\btheta_{i-1}$
        \EndFor
        \vspace{-1mm}
    \State \Return $\btheta_N$
    \EndProcedure
    \vspace{1mm}
    \Procedure{SMC}{$\{\gamma_t(\bx_{1:t})\}_{t=1}^T$, $\{q_{\btheta}(\bx_t \mid \bx_{t-1}, \by_{1:T})\}_{t=1}^T$, K}
      \State See Algorithm \ref{app:algo:smc} in Appendix \ref{app:sec:background:smc}.
    \EndProcedure
    \vspace{-1mm}
    \end{algorithmic}
    \end{spacing}
\end{algorithm}
\vspace{-4mm}
\end{figure}

Note that the optimal backwards twist is proportional to the ratio of a ``backwards message'' $p_{\btheta}(\bx_t \mid \by_{t:1:T})$ and the latent marginal $p_{\btheta}(\bx_t)$ (Equation \ref{equ:methods:bwd_ratio}). Thus, we can learn the backwards twist using density ratio estimation (DRE) \cite{mohamed2016learning, sugiyama2012density}.

DRE via classification estimates the ratio of two densities $a(x)/b(x)$ by training a classifier to distinguish between samples from $a$ and $b$. If such a classifier is trained using the logit link function, then its raw output will approximate $\log a(x) - \log b(x)$ up to a constant \cite{sugiyama2012density}. Using this approach, we interpret $\log r_{\bpsi}(\by_{t+1:T},\bx_t)$ as the logit of a Bernoulli classifier and train it to distinguish between samples from $p_{\btheta}(\bx_t, \by_{t+1:T})$ and $p_{\btheta}(\bx_{t})p_{\btheta}(\by_{t+1:T})$, which are available from the model. When trained in this way, $\log r_{\bpsi}(\by_{t+1:T}, \bx_t)$ will approximate $\log p_{\btheta}(\bx_t~\mid~\by_{t+1:T}) - \log p_{\btheta}(\bx_t)$ up to a constant which can be ignored, for details see Appendix \ref{app:sec:methods:dre} and \cite{sugiyama2012density}.

We use the DRE-learned twisting functions in an alternating scheme that first holds $p_{\btheta}, q_{\btheta}$ fixed and updates $r_{\bpsi}$ using density ratio estimation, and then holds $r_{\bpsi}$ fixed and updates $p_{\btheta}$ and $q_{\btheta}$ by ascending a biased gradient estimator (no resampling terms) of $\mathcal{L}_\mathrm{SIXO}(\btheta, \bpsi)$ in $\btheta$. We call the full alternating procedure for learning $\btheta$ and $\bpsi$ SIXO-DRE, see Algorithm \ref{algo:sixo-dre}.

\subsection{The SIXO Bound Can Become Tight}
\label{sec:methods:tightness}

\citet{maddison2017filtering} show that the FIVO bound can only become tight in models with uncommon dependency structures.  We show that the SIXO bound can become tight for any model in the class defined in Section \ref{sec:background}.
\begin{prop}
\label{prop:sixo} 
{\normalfont Sharpness of the SIXO bound.} Let $p(\bx_{1:T}, \by_{1:T})$ be a latent variable model with Markovian structure as defined in Section \ref{sec:background}, let $\mathcal{Q}$ be the set of possible sequences of proposal distributions indexed by parameters $\btheta \in \Theta$, and let $\mathcal{R}$ be the set of possible sequences of positive, integrable twist functions indexed by parameters $\bpsi \in \Psi$. Assume that $\{p(\bx_t \mid \bx_{t-1}, \by_{1:T})\}_{t=1}^{T} \in \mathcal{Q}$ and $\{p(\by_{t+1:T} \mid \bx_t)\}_{t=1}^{T-1} \in \mathcal{R}$. Finally, assume $\mathcal{L}_\mathrm{SIXO}(\btheta, \bpsi, \by_{1:T})$ has the unique optimizer $\btheta^*, \bpsi^* = \argmax_{\btheta\in \Theta, \bpsi \in \Psi} \mathcal{L}_\mathrm{SIXO}(\btheta, \bpsi, \by_{1:T})$.

Then the following holds:
\begin{enumerate}
    \item $q_{\btheta^*}(\bx_t \mid \bx_{1:t-1}, \by_{1:T}) = p(\bx_t \mid \bx_{1:t-1}, \by_{1:T})$ \ for \ $t=1,\ldots, T$,
    \item $r_{\bpsi^*}(\by_{t+1:T}, \bx_t) \propto p(\by_{t+1:T} \mid \bx_t)$ up to a constant independent of $\bx_t$ \ for \ $t=1,\ldots, T-1$, 
    \item $\mathcal{L}_\mathrm{SIXO}(\btheta^*, \bpsi^*, \by_{1:T}) = \log p(\by_{1:T})$ for any number of particles $K \geq 1$.
\end{enumerate} 
\end{prop}
\begin{proof}
See Appendix \ref{app:sec:methods:sixo_bound:tightness}.
\end{proof}

This is an important advantage of our work---the SIXO objective is the first to recover the true marginal likelihood with a finite number of particles while also being tailored to sequential tasks.
\section{Related Work}
\label{sec:related}

Good references for SMC include \citet{doucet2009tutorial, naesseth2019elements}, and \citet{del2004feynman} which provides a theoretical treatment of a generalization of SMC called Feynman-Kac formulae. \citet{pitt1999filtering} introduced the auxiliary particle filter, an early smoothing SMC method which constructs an estimate of the one-step backwards message $p_{\btheta}(\by_{t+1} \mid \bx_t)$ using simulations from the model. Smoothing SMC in general is discussed thoroughly in \citet{briers2010smoothing} and \citet{del2010forward}. Later, \citet{whiteley2014twisted} introduced twisted particle filters, which perform SMC on a twisted model that approximates the smoothing distributions by multiplying the model's filtering distributions with ``twisting'' functions. Our work extends the theoretical framework in \citet{whiteley2014twisted} by proposing practical and effective methods for learning parametric twisting functions.

To make smoothing SMC computationally tractable, fixed-lag techniques use information from only a fixed window of future observations, as introduced in \citet{bernardo1999fixed} and surveyed in \citet{lin2013lookahead}. For example, \citet{park2020inference} use simulations from the model to estimate fixed-lag twisting functions and \citet{doucet2006efficient} sample blocks of latents conditional on their observations via various Monte Carlo methods. These methods suffer from computational complexity that grows with the window size, and fail to take advantage of all future observations.

Other methods use twisting functions which depend on all observations. Most similar to our approach are \citet{guarniero2017iterated} and \citet{heng2020controlled} which learn parametric twists using a Bellman-type decomposition of the lookahead distributions $p(\by_{t+1:T} \mid \bx_t)$ in terms of the same distributions one step into the future. \citet{del2015sequential} use Gaussian processes to approximate the twists, \citet{lindsten2018graphical} use traditional graphical model techniques such as loopy belief propagation and expectation propagation, and \citet{ruiz2017particle} use optimal control techniques. None of these approaches consider model learning and their twist learning techniques are highly specialized to their problem settings.

Fitting model parameters via stochastic gradient ascent on an evidence lower bound (ELBO) was introduced in \citet{ranganath2013black,hoffman2013stochastic,kingma2013auto}, and later generalized to the Monte Carlo objectives (MCO) framework by \citet{mnih2016variational}. Since then, works have considered optimizing lower bounds defined by the normalizing constant estimators from multiple importance sampling \cite{burda2015importance}, rejection sampling and Hamiltonian Monte Carlo \cite{lawson2019energy}, filtering SMC \cite{maddison2017filtering,naesseth2018variational,le2017auto}, and smoothing SMC \cite{lawson2018twisted,moretti2020variational,moretti2019smoothing}. The prior work on smoothing SMC used an objective defined by forward filtering backwards smoothing \cite{briers2010smoothing} which suffers from the same particle degeneracy issues as filtering SMC and cannot become tight. \citet{kim2020variational} optimize the importance weighted autoencoder (IWAE) bound \cite{burda2015importance} using a gradient estimator that incorporates a \emph{baseline} derived from future likelihood estimates, but do not use SMC or resampling in their bound.
\section{Experiments}
\label{sec:exp}

We experimentally explore our claims that:
\begin{enumerate}
    \item The SIXO bound can become tight while FIVO cannot.
    \item DRE-learned twists enable better posterior inference than filtering SMC.
    \item Model learning with SIXO provides better parameter estimates than FIVO.
\end{enumerate}

\subsection{Gaussian Drift Diffusion}
\label{sec:exp:gdd}

We first consider a one-dimensional Gaussian drift diffusion process with joint distribution
\begin{align}
    p_{\btheta} \left(\bx_{1:T}, y_T \right) = \cn \left(y_T \mid x_T + \alpha, \sigma_y^2 \right) \cn \left(x_1 ; \alpha, \sigma_x^2 \right) \prod_{t=2}^T \cn \left(x_{t} \mid x_{t-1} + \alpha, \sigma_x^2 \right).
\end{align}
The single free model parameter is the drift $\alpha$, the state is $x_t \in \mathbb{R}$, and the observation is $y_T \in \mathbb{R}$. Figures \ref{fig:banner:dists:filtering} and \ref{fig:banner:dists:smoothing} show that for $\alpha=0$ the filtering and smoothing distributions in this model quickly diverge, which can lead to poor inference for filtering methods.

We compare joint model, proposal and twist learning using two variants of SIXO to variational inference with the IWAE bound~\citep{burda2015importance} and FIVO with unbiased gradients \cite{maddison2017filtering,naesseth2018variational,le2017auto}. All methods use an independent proposal at each time step parameterized as $q_t(x_t \mid x_{t-1}, y_T) = \mathcal{N}(x_t; f_t(x_{t-1}, y_T), \sigma_{qt}^2)$ where $f_t$ is an affine function, a family which contains the optimal proposal. SIXO-u uses twists parameterized as $r_t(y_T, x_t) = \mathcal{N}(y_T; g_t(x_t), \sigma_{rt}^2)$ where $g_t$ is an affine function, a family which contains the true lookahead distributions. The SIXO-DRE twist $\log r_t(y_T, x_t)$ is parameterized as a quadratic function of $x_t$, where the parameters of the quadratic function are generated by a neural network with inputs $(y_T, t)$. The true log density ratio will be quadratic in $\bx_t$, so if the neural network is sufficiently flexible, the true log density ratio function can be obtained.

\begin{figure}[t]
    \centering
    
    \begin{subfigure}[t]{0.5\textwidth}
        \includegraphics[width=\textwidth]{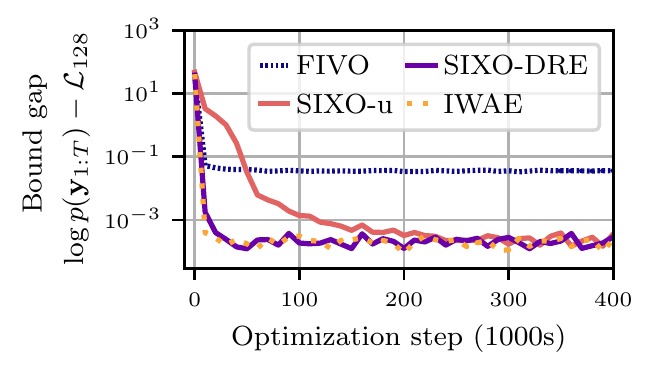}
        \vspace*{-0.5cm}
        \caption{Bound gap for different methods. }
        \label{fig:exp:gdd:results:bound}
    \end{subfigure}%
    \hfill%
    \begin{subfigure}[t]{0.5\textwidth}
        \includegraphics[width=\textwidth]{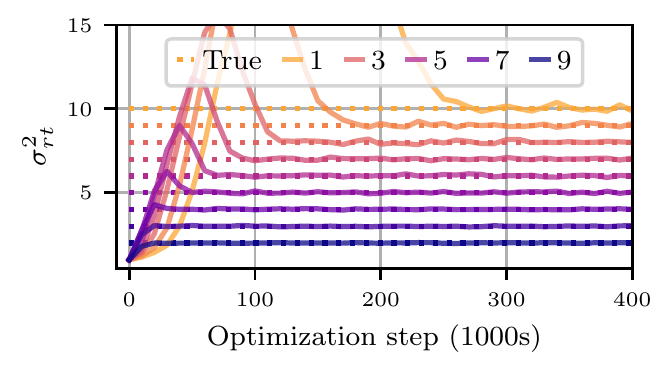}
        \vspace*{-0.5cm}
        \caption{Convergence of twist parameters with SIXO-u.}
        \label{fig:exp:gdd:results:param}
    \end{subfigure}%
    \caption{Bound gap and parameter convergence for the Gaussian drift diffusion experiment presented in Section \ref{sec:exp:gdd}.  Further figures and discussion are included in Appendix \ref{app:sec:experiments:gdd}.}
    \label{fig:exp:gdd:results}
\end{figure}

Figure \ref{fig:exp:gdd:results:bound} shows the convergence of the variational bound for each method. As expected IWAE recovers a tight variational bound, whereas FIVO does not. While SIXO-u does recover a tight variational bound, the high variance of the unbiased gradient estimator makes it impractical for non-toy problems. Conversely, SIXO-DRE achieves a tight bound but under biased, lower variance gradients. This motivates its use in more complex, non-linear settings where the unbiased FIVO gradients are not practical. Figure \ref{fig:exp:gdd:results:param} shows that SIXO-u recovers the correct twist parameters. More figures illustrating the convergence of $\btheta$ and $\bpsi$ are included in Appendix \ref{app:sec:experiments:gdd}.

In Figures \ref{fig:banner:lineages:fivo} and \ref{fig:banner:lineages:sixo} we compare particle trajectories under FIVO and SIXO-u. We see that FIVO consistently proposes particles with high likelihood under the posterior distributions (identical to the smoothing distributions in this case) which are discarded by the resampling steps in filtering SMC. In contrast, SIXO both proposes particles with high posterior likelihood and retains them through the resampling steps by properly scoring particles under the twisted target distributions. These results empirically verify the theoretical claims made in Section \ref{sec:methods:tightness}.

\subsection{Stochastic Volatility Model}
\label{sec:exp:svm}
We now apply SIXO to a stochastic volatility model (SVM) of monthly foreign exchange rates for $N=22$ currencies in the period from 9/2007 to 8/2017~\citep{chib2009multivariate}. The SVM generative model is
\begin{align}
    \bx_1 \sim \mathcal{N} \left( \mathbf{0}, \mathbf{Q} \right) & , \quad \bx_t = \bmu + \bphi\odot \left( \bx_{t-1} - \bmu \right) + \boldsymbol\nu_t , \quad \by_t = \boldsymbol\beta \odot \mathrm{exp} \left( \frac{\bx_t}{2} \right) \odot \mathbf{e}_t , 
\end{align}
with transition noise $\boldsymbol\nu_t \sim \mathcal{N} \left( \mathbf{0}, \mathbf{Q} \right)$, observation noise $\mathbf{e}_t \sim \mathcal{N} \left( \mathbf{0}, I_{N\times N} \right)$, states $\bx_{1:T} \in \mathbb{R}^{T\times N}$, and observations $\by_{1:T} \in \mathbb{R}^{T \times N}$.
All multiplications are performed element-wise as represented by $\odot$. The model has free parameters $\bmu \in \mathbb{R}^N , \bphi \in \left[ 0, 1 \right]^N , \boldsymbol\beta \in \mathbb{R}^N_+$, and $\mathbf{Q} \in \mathrm{diag}(\mathbb{R}^N_+)$ such that there are $4N$ model parameters. The proposal, $q_{\btheta}$, is structured as $q_{\btheta}(\bx_{1:T}) \propto \prod_{t=1}^T \mathcal{N}(\bx_t;\boldsymbol\mu_t, \mathbf{\Sigma}_t) p_{\btheta}(\bx_t \mid \bx_{t-1})$ with means $\boldsymbol\mu_t \in \mathbb{R}^N$ and diagonal covariance matrices $\mathbf{\Sigma}_t \in \mathrm{diag}(\mathbb{R}_+^N)$ so that there are $2NT$ proposal parameters. We compare three approaches: FIVO, SIXO with quadrature twist (SIXO-q), and SIXO with density ratio twist (SIXO-DRE).  For more specifics and hyperparameters, see Appendix \ref{app:sec:experiments:svm}. 

\paragraph{Train Performance} We first compare our methods in terms of log marginal likelihood lower bounds as in \citet{naesseth2018variational}. We evaluate all checkpoints after $75\%$ of training using each method's corresponding 4-particle bound; for FIVO we report the FIVO bound, for SIXO-q we report the SIXO-q bound, and for SIXO-DRE we report the SIXO-DRE bound. Even though SIXO-q only scores a single future observation, it still obtains a 7-nat improvement over FIVO. SIXO-DRE, meanwhile, conditions on all future observations and obtains a 10-nat improvement over FIVO.

We also attempt to estimate the true marginal likelihood by computing a bootstrap particle filter's log marginal lower bound with 2048 particles, denoted $\mathcal{L}_\mathrm{BPF}^{2048}$ \cite{gordon1993bpf}. Interestingly, a one-way ANOVA \cite{fisher1992statistical} does not reject the null hypothesis that the training set $\mathcal{L}_\mathrm{BPF}^{2048}$ means are all equal ($p = 0.25$), suggesting that the log marginal likelihoods on training data are indistinguishable and training performance has saturated. It is clear, however, that SIXO performs better inference and makes more efficient use of particles as the $\mathcal{L}_\mathrm{SIXO}^4$ bounds are significantly higher than $\mathcal{L}_\mathrm{FIVO}^4$ for models with similar true marginal likelihoods.

\begin{table}
      \caption{Performance of FIVO and SIXO on the SVM.}
      \label{tab:svm}
      \centering
    \begin{tabular}{@{}lccc@{}}
    \toprule
    Method &  Train $\mathcal{L}_\mathrm{Method}^4$ (as in \cite{naesseth2018variational}) & Train $\mathcal{L}_\mathrm{BPF}^{2048}$ & Test $\mathcal{L}_\mathrm{BPF}^{2048}$      \\ \midrule
    FIVO        & $6921.29 \pm 1.33$            & $\mathbf{7020.14 \pm  2.86}$      & $3352.71 \pm 1.3$ \\ 
    SIXO-q      & $6928.90 \pm 1.24$            & $\mathbf{7019.65 \pm 2.97}$       & $3353.45 \pm 1.58$\\ 
    SIXO-DRE    & $\mathbf{6931.51 \pm 2.08}$   & $\mathbf{7019.42 \pm 3.01}$       & $\mathbf{3354.27 \pm 1.60}$ \\ \bottomrule
    \end{tabular}
\end{table}

\paragraph{Test Performance} We also compare methods on a held-out test set to evaluate each method's influence on model learning. We construct this test set using the same data source as the training set, but use the period of time since \citet{naesseth2018variational} was published (an extra $55$ months). Again, we report BPF log marginal lower bounds with 2048 particles and find that SIXO-DRE outperforms SIXO-q and FIVO.

\subsection{Hodgkin-Huxley Model}
\label{sec:exp:hh}

\begin{figure}[t]
    \centering
    \includegraphics[width=\textwidth]{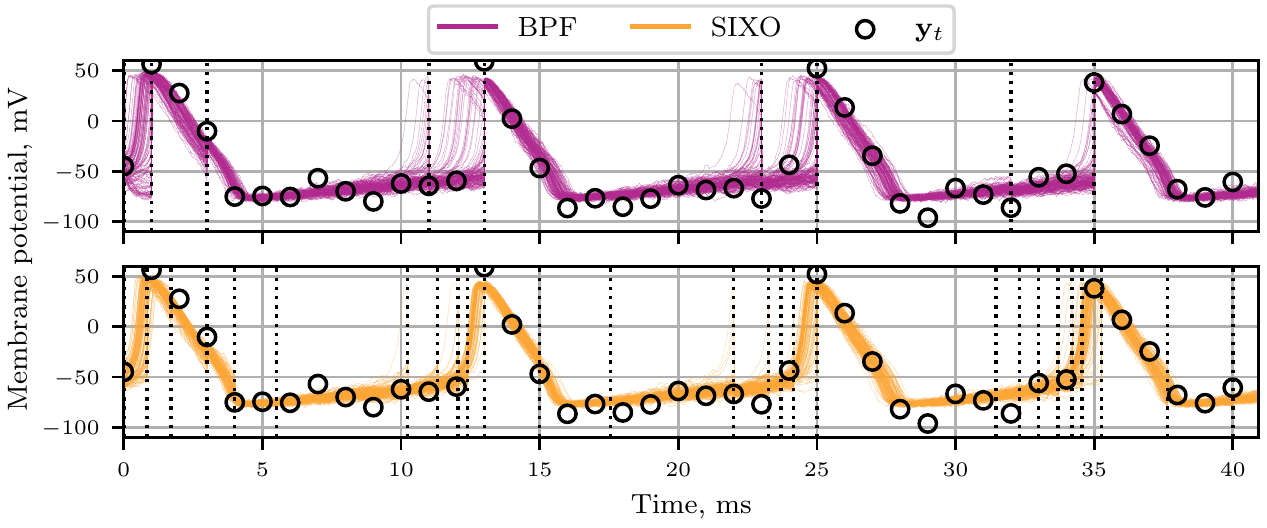}
    \caption{Comparison of the filtering distributions generated by a bootstrap particle filter (BPF) (top) and a SIXO sweep (bottom) on synthetic data from the Hodgkin-Huxley model.  Dotted vertical lines are resampling events. Both sweeps use true model parameters and a bootstrap proposal, but SIXO uses DRE-trained twisting functions. We see that the twist has reduced the number of erroneous spikes generated under the BPF, and more particles accurately predict the initiation of a spike.}
    \label{fig:exp:hh:twisted_smc}
\end{figure}

We conclude by comparing FIVO and SIXO on the Hodgkin-Huxley (HH) model of neural action potentials~\citep{hodgkin1952quantitative, dayan2005theoretical}. A single neuron is represented with a four-dimensional state-space: the instantaneous membrane potential and the relative conductivity of three ion gates. A noise-corrupted and subsampled membrane potential can be obtained using electrodes~\citep{kita2008microelectrodes} or voltage imaging~\citep{peterka2011imaging}. The state of the gates, however, is not observable, and must be inferred from the noisy potential recordings. The physiological parameters governing the time-evolution of the system are also of interest, such as the base conductance of each of the ion channels.

We implement the HH model as a four-dimensional nonlinear state space model with Gaussian transition noise~\citep{huys2009smoothing}. The observation is a single Gaussian-distributed value with mean equal to the instantaneous potential. We subsample observations by a factor of $50$ to simulate an acquisition frequency of $1$kHz. For more details, see Appendix \ref{app:sec:experiments:hh}.
 
In this model action potentials, or spikes, are rare events that happen quickly and invoke a rapid change in the state. Therefore, filtering-based inference is particularly disadvantageous as noisy observations may trigger erroneous spikes or ``miss'' true spikes. 

\paragraph{Inference}
Figure \ref{fig:exp:hh:twisted_smc} shows a BPF generating spurious spikes and missing the initiation of other spikes, demonstrating this shortcoming. SIXO, despite using an unlearned bootstrap proposal, generates fewer spurious spikes and fewer particles miss spikes. SIXO also achieves a higher log marginal lower bound ($-73.88$ nats) than the bootstrap particle filter ($-74.89$ nats), showing that it performs more effective inference. These results hint at broader potential for DRE twists as a flexible, general-purpose tool for improving inference in non-linear models.

\begin{figure}[t]
    \centering
    
    \begin{subfigure}[t]{0.48\textwidth}
        \includegraphics[width=\textwidth]{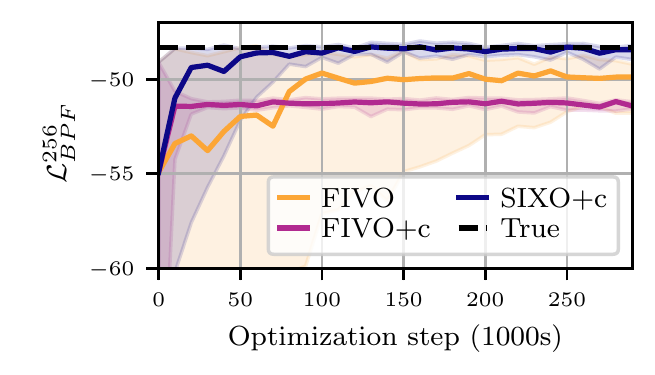}
        \caption{Test set $\mathcal{L}_\mathrm{BPF}^{256}$ over training.}
        \label{fig:hh_results:ll}
    \end{subfigure}%
    \hfill%
    \begin{subfigure}[t]{0.48\textwidth}
        \includegraphics[width=\textwidth]{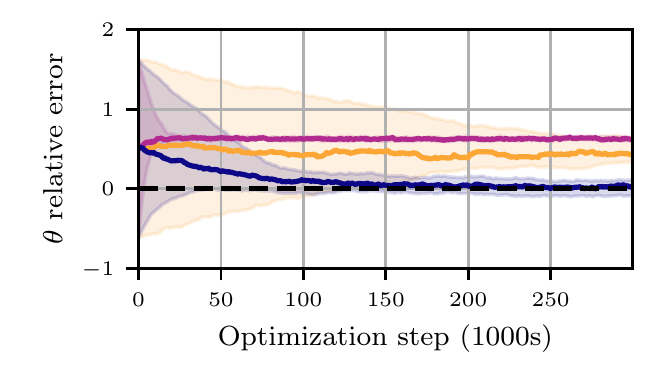}
        \caption{Relative parameter error over training.}
        \label{fig:hh_results:theta}
    \end{subfigure}%
    
    \caption{HH model learning results. SIXO is the only method to recover the true parameter.}
    \label{fig:hh_results}
\end{figure}

\paragraph{Model Learning}
We conclude by comparing FIVO and SIXO for parameter recovery in Figure \ref{fig:hh_results}. The relative parameter error is computed as $(\btheta^* - \btheta^{\mathrm{true}}) / \btheta^{\mathrm{true}}$. We found gradient clipping improves stability in SIXO, so we compare to FIVO with and without clipping. FIVO with clipping converges quickly to an incorrect parameter that achieves a low marginal likelihood, while FIVO without clipping converges more slowly to parameters that achieves a higher bound. SIXO is the only method to recover the correct model, and achieves the highest log marginal likelihood.

\begin{table}[t]
      \caption{Performance of FIVO and SIXO on the Hodgkin-Huxley model.}
      \label{tab:hh}
      \centering
    \begin{tabular}{@{}lccc@{}}
    \toprule
    Method & Test $\mathcal{L}_\mathrm{Method}^{256}$ & Test $\mathcal{L}_{\mathrm{BPF}}^{256}$ & Relative Parameter Error\\ 
    \midrule
    (True model) & (N/A) & $(-48.31)$ & $(0.0 \pm 0.0)$ \\
    \midrule
    FIVO            & $-50.25 \pm 1.38$                     & $-50.21 \pm 1.28$                     & $0.46 \pm 0.16$\\
    FIVO (+ clip)   & $-51.43 \pm 0.49$                     & $-51.35 \pm 0.29$                     & $0.63 \pm 0.02$\\
    \textbf{SIXO}   & $\mathbf{-47.94} \pm \mathbf{0.25}$   & $\mathbf{-48.53} \pm \mathbf{0.36}$   & $\mathbf{0.02} \pm \mathbf{0.09}$\\
    \bottomrule
    \end{tabular}
\end{table}

\section{Conclusion}
\label{sec:conclusion}

In this work we proposed a method of learning twisting functions for smoothing SMC via density ratio estimation. Our approach involves ascending a lower bound on the log marginal likelihood that can theoretically become tight, a first for sequential Monte Carlo objectives. We verified our theoretical claims by experimentally demonstrating improvements over existing techniques in inference and model learning.

\paragraph{Acknowledgements and Disclosure of Funding} We thank Matt MacKay for helpful discussions and edits. This work was supported by grants from the Simons Collaboration on the Global Brain (SCGB 697092), the NIH BRAIN Initiative (U19NS113201 and R01NS113119). Some of the computation for this work was made possible by Microsoft Education Azure cloud credits. Dieterich Lawson was supported in part by a Stanford Data Science Fellowship.

\bibliographystyle{unsrtnat}
\bibliography{99_main.bib}
\appendix
\section*{Appendices for \bsixo: Smoothing Inference with Twisted Objectives}
\section{Table of Notation}
\label{sec:app:notation}

\renewcommand{\arraystretch}{1.2}
\definecolor{Gray}{gray}{0.95}
\begin{table}[H]
\footnotesize
\begin{adjustwidth}{-3in}{-3in}  
\begin{center}

\begin{tabular}{p{4.7cm}p{3.35cm}p{6cm}}
\toprule
\textbf{Name} & \textbf{Symbol} &\textbf{Notes} \\ \midrule
Sequence length & $T$ & $T \in \mathbb{N}$ \\
\rowcolor{Gray}
Timestep & $t$ & $t \in \{1, \ldots, T\}$ \\
Latent state & $\bx_t$ & $\bx_t \in \cx$ \\
\rowcolor{Gray}
Observation & $\by_t$ & $\by_t \in \cy$ \\
Observation sequence & $\by_{1:T}$ & $\by_{1:T} \in \cy^T$ \\
\rowcolor{Gray}
Number of SMC particles & $K$ & $K \in \mathbb{N}$\\
$k$\textsuperscript{th} particle latent trajectory & $\bx_{1:t}^{k}$ & $\bx_{1:t}^k \in \cx^t$ \\ \midrule %
\rowcolor{Gray}
Model and proposal parameters & $\btheta$ & $\btheta \in \BTheta$ \\
Twist parameters & $\bpsi$ & $\bpsi \in \BPsi$ \\ %
\rowcolor{Gray}
Joint distribution & $p_{\btheta}( \bx_{1:T}, \by_{1:T} )$ & Distribution on $\cx^T \times \cy^T$\\

$t$\textsuperscript{th} transition distribution & $p_{\btheta}( \bx_t \mid  \bx_{t-1} )$ & Conditional distribution on $\cx$\\
\rowcolor{Gray}
$t$\textsuperscript{th} observation distribution & $p_{\btheta}( \by_t \mid  \bx_{t} )$ &  Conditional distribution on $\cy$\\
Proposal distributions & $\{q_{\btheta} ( \bx_t \mid \bx_{t-1}, \by_{1:T})\}_{t=1}^T$ & $T$ conditional distributions on $\cx$\\
\rowcolor{Gray}
Twist functions & $\{r_{\bpsi} ( \by_{t+1:T}, \bx_t)\}_{t=1}^{T-1}$ & $T-1$ positive integrable functions in $\cx \times \cy^{T-t-1} \times \boldsymbol\Psi \rightarrow \mathbb{R}_{\geq 0}$ \\ %
Filtering distributions & $\{p_{\btheta}( \bx_{1:t} \mid \by_{1:t} )\}_{t=1}^T$  & $T$ conditional distributions on $\cx^t$\\
\rowcolor{Gray}
Smoothing distribution & $\{p_{\btheta}( \bx_{1:t} \mid \by_{1:T})\}_{t=1}^T$ & $T$ conditional distributions on $\cx^t$\\

Unnormalized target distributions & $\left\lbrace \gamma_t(\bx_{1:t}) \right\rbrace_{t=1}^T$ & $T$ positive integrable functions in $\cx^t \rightarrow \mathbb{R}_{\geq 0}$\\
\rowcolor{Gray}
Normalized target distributions & $\left\lbrace \pi_t(\bx_{1:t}) \right\rbrace_{t=1}^T$ & $T$ distributions on $\cx^t$ \\
$t$\textsuperscript{th} normalizing constant & $Z_t$ & Positive real, $\pi_t(\bx_{1:t}) = \gamma_t(\bx_{1:t})/Z_t$ \\ \midrule %
\rowcolor{Gray}
Bootstrap particle filter & BPF & SMC using $p(\bx_t \mid \bx_{t-1})$ as the proposal and no twist. \\
Filtering sequential Monte Carlo & FSMC & SMC with filtering distributions as targets. \\
\rowcolor{Gray}
SIXO-unified & SIXO-u & SIXO objective optimized with unbiased gradient ascent. \\
SIXO-quadrature & SIXO-q & SIXO objective with quadrature twist. \\
\rowcolor{Gray}
SIXO-density ratio estimation & SIXO-DRE & SIXO with density ratio estimate twist \\ \midrule %
BPF bound & $\mathcal{L}_{\mathrm{BPF}}^K$ & Log marginal likelihood bound from a $K$-particle BPF. \\
\rowcolor{Gray}
FIVO bound & $\mathcal{L}_{\mathrm{FIVO}}^K$ & FIVO bound \eqref{eq:methods:fivo} with $K$ particles. \\
SIXO bound & $\mathcal{L}_{\mathrm{SIXO}}^K$ & SIXO bound \eqref{equ:sixo_bound} with $K$ particles.\\
\rowcolor{Gray}
DRE loss & $\mathcal{L}_{\mathrm{DRE}}$ & Loss used to learn twist with DRE. \\
\bottomrule
\end{tabular}

\end{center}
\end{adjustwidth}
\end{table}

\newpage

\section{Background}
\label{app:sec:background}

\subsection{Sequential Monte Carlo}
\label{app:sec:background:smc}
Sequential Monte Carlo (SMC) is a popular method for sampling from posterior distributions with sequential structure. For a thorough introduction we refer the reader to \citet{doucet2009tutorial} and \citet{naesseth2019elements}.  We reproduce the general SMC algorithm in Algorithm \ref{app:algo:smc}.

\begin{algorithm}
    \caption{Sequential Monte Carlo}
    \begin{spacing}{1.25}
    \begin{algorithmic}[1]
        \Procedure{SMC}{$\{\gamma_t(\bx_{1:t})\}_{t=1}^{T}$, $\{q_t(\bx_t \mid \bx_{1:t-1})\}_{t=1}^T$, $K$}
        \State $w_0^{1:K} = 1,\hspace{2mm} \widehat{Z}_0=1$
        \For {$t = 1,\ldots, T$}
            \For {$k = 1, \ldots, K$}
                \State $\bx_t^k \sim q_t(\bx_t | \bx_{1:t-1}^k)$
                \State $\bx_{1:t}^k = (\bx_{1:t-1}^k, \bx_t^k)$
                \State $\alpha_t^k = \dfrac{\gamma_t(\bx_{1:t}^k)}{\gamma_t(\bx_{1:t-1}^k)q_t(\bx_t^k \mid \bx_{1:t-1}^k)}$
                \State $w_t^k = w_{t-1}^k\alpha_t^k$
            \EndFor
            \State $\widehat{Z_t/Z_{t-1}} = \dfrac{\sum_{k=1}^K w_{t}^k}{\sum_{k=1}^K w_{t-1}^k}$
            \State $\widehat{Z}_t = \widehat{Z}_{t-1}(\widehat{Z_{t}/Z_{t-1}})$
            \If{should resample}
                \For{$k = 1, \ldots, K$}
                    \State $a_t^k \sim \text{Categorical}(w_t^{1:K})$
                    \State $\tilde{\bx}_{1:t}^k = \bx_{1:t}^{a_t^k}$
                \EndFor
                \State $\bx_{1:t}^{1:K} = \tilde{\bx}_{1:t}^{1:K}$
                \State $w_t^{1:K} = 1$
            \EndIf
        \EndFor
        \State\Return $\widehat{Z}_{T}$, $\bx_{1:T}^{1:K}$
        \EndProcedure
    \end{algorithmic}
    \end{spacing}
    \label{app:algo:smc}
\end{algorithm}

\subsection{Factoring the Smoothing Distributions}
\label{app:sec:background:twist-factor}
Here we show that the smoothing distributions can be factored into the filtering distributions and the lookahead distributions. 

Let $p_{\btheta}(\bx_{1:T}, \by_{1:T})$ be a model as defined in \eqref{equ:ssm}. Then
\begin{align}
    p_{\btheta}(\bx_{1:t}, \by_{1:T}) =& p_{\btheta}(\bx_{1:t}, \by_{1:t})p_{\btheta}(\by_{t+1:T} \mid \by_{1:t}, \bx_{1:t})\\
    =& p_{\btheta}(\bx_{1:t}, \by_{1:t})p_{\btheta}(\by_{t+1:T} \mid \bx_{t}),
\end{align}
where $\by_{t+1:T}$ is conditionally independent of $(\by_{1:t}, \bx_{1:t-1})$ given $\bx_t$ because of the Markov structure of $p_{\btheta}$.
\section{Methods}
\label{app:sec:methods}

\subsection{The Gradients of SMC}
\label{app:sec:methods:unified-grads}
The original FIVO papers~\citep{maddison2017filtering, naesseth2018variational, le2017auto} used biased gradients to optimize $\mathcal{L}_\mathrm{FSMC}$ which ignore score-function gradient terms arising from the discrete resampling operations. We use the same biased gradient estimator in SIXO-DRE when optimizing $\mathcal{L}_\mathrm{SIXO}(\btheta, \bpsi, \by_{1:T})$ in terms of $\btheta$, but for SIXO-u we use the full unbiased gradient estimator. We derive both estimators here.

Assume resampling occurs at each timestep, let $\mathbf{A}_t = a_t^{1:K}$ and $\BX_t = \bx_t^{1:K}$ be the ancestor indices and latents for all particles at time $t$, and let $\BA = (\BA_1, \ldots, \BA_{T-1})$ and $\BX = (\BX_1, \ldots, \BX_T)$ be the full sequence of ancestor indices and latents. We can then write the probability distribution over $\BX, \BA$ that defines SMC as
\begin{align}
p_\mathrm{SMC}(\BX,\BA) = \left(\prod_{k=1}^K q(\bx_1^k)\right) \prod_{t=2}^{T}\prod_{k=1}^K q(\bx_t^k|\bx_{1:t-1}^{a_{t-1}^k})\overline{\alpha}_{t-1}^{a_{t-1}^k}
\end{align}
where $\overline{\alpha}_t^i = \alpha_t^i/(\sum_{k=1}^K \alpha_t^k)$ is the normalized incremental weight. When the proposals and targets of SMC are parametric functions of $\btheta$, both $q$ and $\overline{\alpha}$ will depend on $\btheta$.

To emphasize its dependence on $\btheta$ when run with parametric twists and proposals we will write $p_\mathrm{SMC}(\BX, \BA)$ as $p(\BX, \BA; \btheta)$. Then the gradient of $\mathcal{L}_\mathrm{SIXO}$ is defined as
\begin{align}
    \nabla_{\btheta}(\mathcal{L}_\mathrm{SIXO}(\btheta)) =& \nabla_{\btheta} \E_{\BX,\BA \sim p(\BX, \BA; \btheta)}[\log \widehat{Z}(\BX, \BA, \btheta)]
\end{align}
where we have rewritten $\widehat{Z}_\mathrm{SIXO}(\btheta, \bpsi, \by_{1:T})$ as $\widehat{Z}(\BX,\BA,\btheta)$ to emphasize its dependence on the random variables $\BX$ and $\BA$ and suppress its dependence on $\bpsi$ and $\by_{1:T}$.

The first step is to reparameterize the expectation in terms of continuous noise instead of $\BX$~\citep{price1958useful,bonnet1964transformations,kingma2013auto,salimans2013fixed,rezende2014stochastic}. Assume $\BX$ is from a reparameterizable distribution and let $\phi_{\BX}(\btheta, \epsilon)$ be a function that deterministically combines continuous noise $\epsilon$ and the parameters $\btheta$ to produce a sample $\BX$. Then we have
\begin{align}
  \nabla_{\btheta} \E_{\BX,\BA}[\log \widehat{Z}(\BX, \BA, \btheta)] = \nabla_{\btheta} \E_{\epsilon,\BA}[\log \widehat{Z}(\phi_{\BX}(\btheta, \epsilon), \BA, \btheta)].\label{eq:app:reparam}
\end{align}
We will further abuse notation by writing $\widehat{Z}(\phi_{\BX}(\btheta, \epsilon), \BA, \btheta)$ as $\widehat{Z}(\epsilon, \BA, \btheta)$. Rewriting the expectation (\ref{eq:app:reparam}) as an integral gives
\begin{align}
    \nabla_{\btheta} \E_{\epsilon,\BA}[\log \hat{Z}(\epsilon, \BA, \btheta)] =& \nabla_{\btheta} \int \log \hat{Z}(\epsilon, \BA, \btheta)p(\epsilon, \BA; \btheta) d\epsilon d\BA.
\end{align}
Assuming that we can differentiate under the integral allows us to break the integrand apart using the product rule as
\begin{align}
    \int \nabla_{\btheta} (\log \widehat{Z}(\epsilon, \BA, \btheta))p(\epsilon, \BA; \btheta) +  \log \widehat{Z}(\epsilon, \BA, \btheta)\nabla_{\btheta} (p(\epsilon, \BA; \btheta))d\epsilon d\BA. \label{eq:app:product}
\end{align}
The left hand term in the integrand of (\ref{eq:app:product}) equals $\E_{\epsilon,\BA}[\nabla_{\btheta} \log \widehat{Z}(\epsilon, \BA, \btheta)]$, the expectation of a gradient that can be estimated using simple Monte Carlo. The right hand term is a ``score-function gradient''~\citep{williams1992simple} which can be rewritten using the fact that $\nabla_{\btheta}(\log f(\btheta)) = \nabla_{\btheta}(f(\btheta)) / f(\btheta)$ as
\begin{align}
    \int \log \widehat{Z}(\epsilon, \BA, \btheta)\nabla_{\btheta} (\log p(\epsilon, \BA; \btheta))p(\epsilon, \BA; \btheta)d\epsilon d\BA
\end{align}
which in turn equals the expectation
\begin{align}
\E_{\epsilon,\BA}\left[ \log \widehat{Z}(\epsilon, \BA, \btheta)\nabla_{\btheta} \log p(\epsilon, \BA; \btheta)\right].\label{eq:app:score-function}
\end{align}
Writing both terms together gives the full unbiased gradient that is amenable to estimation with simple Monte Carlo,
\begin{align}
    \E_{\epsilon,\BA}\left[\nabla_{\btheta} \log \widehat{Z}(\epsilon, \BA, \btheta) + \log \widehat{Z}(\epsilon, \BA, \btheta)\nabla_{\btheta} \log p(\epsilon, \BA; \btheta)\right].\label{eq:app:unbiased-grad}
\end{align}
Similar to prior work~\citep{maddison2017filtering,naesseth2018variational,le2017auto,lawson2018twisted} we find that the term on the right hand side of (\ref{eq:app:unbiased-grad}) has prohibitively high variance which inhibits learning. Dropping it gives the biased SMC gradient estimator used in SIXO-DRE,
\begin{align}
    \E_{\epsilon,\BA}[\nabla_{\btheta} \log \widehat{Z}(\epsilon, \BA, \btheta)],\label{eq:app:biased-grad}
\end{align}
which can be estimated using open-source autodiff software~\citep{jax2018github}.

The derivation above is adaptable for any resampling schedule that does not depend on the parameters (and by extension, the weights), but many common resampling schemes such as effective sample size resampling do not meet this requirement. If the resampling scheme depends on the parameters of the model, it introduces additional gradient terms which are not described here. Thus for SIXO-u experiments which use the full unbiased gradient (\ref{eq:app:unbiased-grad}), we use a fixed resampling schedule.

\subsection{Density Ratio Estimation}
\label{app:sec:methods:dre}
Density ratio estimation (DRE) considers estimating ratios of densities, e.g. $a(x)/b(x)$ with $a(x)$ and $b(x)$ defined on the same probability space and $b(x) > 0$ for all $x$. Instead of estimating $a(x)$ and $b(x)$ individually and then forming the ratio, an alternative approach is to directly estimate the \emph{odds} that a given sample of $x$ was drawn from $a$.

Let $p(x,z) = p(z)p(x \mid z)$ be an expanded generative model for $x$ defined as
\begin{align}
    z &\sim \mathrm{Bernoulli}(\alpha) , \\
    x &\sim a(x) \quad \mathrm{if} \quad z=1 , \\
    x &\sim b(x) \quad \mathrm{if} \quad z=0  
\end{align}
with $\alpha \in (0,1)$. We can now write the density ratio in terms of conditionals in this generative model,
\begin{align}
    a(x)/b(x) &= p(x \mid z=1)/p(x \mid z=0)\\
    &= \left(\frac{p(x)p(z=1 \mid x)}{p(z=1)}\right)/\left(\frac{p(x)p(z=0 \mid x)}{p(z=0)}\right) , \\
    &= \left(\frac{p(z=0)}{p(z=1)}\right)/\left(\frac{p(z=0 \mid x)}{p(z=1 \mid x)}\right) , \\
    &= \left(\frac{1-\alpha}{\alpha}\right) \left(\frac{p(z=1 \mid x)}{p(z=0 \mid x)}\right) .
\end{align}
Thus, the density ratio can be rewritten as proportional to the odds that $x$ was drawn from $a(x)$ instead of $b(x)$. 

\textbf{Density ratio estimation via classification} suggests training a binary classifier with supervised learning to predict $z$ given $x$ \cite{sugiyama2012density,mohamed2016learning}. Let $\sigma(x) =1/(1+e^{-x})$ be the sigmoid function and let $g_{\bpsi}(x)$ be a classifier trained with Bernoulli loss to maximize the log probability of a dataset $z_{1:N}, x_{1:N}$ sampled IID from $p(x, z)$. Specifically, $\bpsi$ is fit by minimizing $\mathcal{L}_\mathrm{DRE}(\bpsi)$, defined as
\begin{align}
\mathcal{L}_\mathrm{DRE}(\bpsi) \triangleq &\mathbb{E}_{z_{1:N}, x_{1:N} \sim p(x,z)}\left[\log \prod_{i=1}^N \mathrm{Bernoulli}(z ; \sigma(g_{\bpsi}(x))\right] , \\
= &\mathbb{E}_{z_{1:N}, x_{1:N} \sim p(x,z)}\left[\sum_{i=1}^N z_i\log(\sigma(g_{\bpsi}(x_i))) + (1-z_i)\log(\sigma(g_{\bpsi}(x_i)))\right] . 
\end{align}

If trained in this way, the raw output of $g_{\bpsi}(x)$ will approximate the log-odds that $x$ came from $a(x)$ instead of $b(x)$, i.e.
\begin{align}
    g_{\bpsi}(x) \approx \log \left(\frac{p(z=1 \mid x)}{1- p(z=1 \mid x)} \right) =  \log \left(\frac{p(z=1 \mid x)}{p(z=0 \mid x)} \right).  \label{app:equ:dre_logit_background}
\end{align}
The log of the density ratio can then be expressed as
\begin{align}
    \log a(x)/b(x) &=  \log(1-\alpha) - \log(\alpha) + \log\left(\frac{p(z=1 \mid x)}{p(z=0 \mid x)}\right) , \\
    &\approx \log(1-\alpha) - \log(\alpha) + g_{\bpsi}(x).
\end{align}

Assuming a fixed $\alpha$ parameter, the log-ratio of densities is then proportional to the logit produced by $g_{\bpsi}$ up to an additive constant. Thus as long as we can sample training pairs $(x, z)$ from the expanded generative model above, we can estimate ratios of densities by training a binary classifier.

\subsection{Alternating Density Ratio Twist Training}
\label{app:sec:methods:dre-twist}
To train the density ratio twist functions we define a supervised maximum likelihood update that is applied offline from the update to the model and proposal parameters.  This update is shown in Algorithm \ref{algo:sixo-dre}, labeled as DRE.

To define this update, we use the approach in Section \ref{app:sec:methods:dre} and set $a(x) = p_{\btheta}(\bx_{t} \mid \by_{t+1:T})$ and $b(x) = p_{\btheta}(\bx_{t})$.  To generate positive and negative examples for DRE, we first sample a set of $M$ latent state and observation trajectories from the generative model, $\bx_{1:T}^{1:M}, \by_{1:T}^{1:M} \sim p_{\btheta}(\bx_{1:T}, \by_{1:T})$.  Because these are samples from joint distribution they are also samples from the conditionals $\{p_{\btheta}(\bx_t \mid \by_{t+1:T})\}_{t=1}^{T-1}$.  We will refer to these samples as \emph{positive samples}.  We can then draw a second set of samples, but discard the observed data, $\tilde{\mathbf{x}}_{1:T}^{1:M} \sim p_{\btheta}(\bx_{1:T})$.  We will refer to these as \emph{negative samples}.  The sets of positive and negative examples form the data on which we will train the twist classifier. Generating examples sequentially in this manner is cheap and parallelizable, allowing us to use relatively large values of $M$.  

To train the RNN twist, we first pass the RNN \emph{backwards} over an observed data sample $\by_{1:T}^m$ to generate a sequence of encodings $\mathbf{e}_{1:T-1}^m$ (noting that we flip the resulting encodings so they are ``forward'' in time). To evaluate the probability of a positive classification we concatenate the encoding $\mathbf{e}_t^m$ with $\bx^m_{t}$ at each timestep, feed the result into a multi-layer perceptron (MLP) with scalar output, and take the output as a positive example Bernoulli logit. To evaluate the probability of a negative classification we take the same sequence of encodings, $\mathbf{e}_{1:T-1}$, concatenate them with $\tilde{\bx}^m_{1:T}$, feed the result at each timestep into the same MLP, and take the result as a negative example Bernoulli logit. These outputs are used to compute the cross-entropy loss as written in Algorithm \ref{algo:sixo-dre} at each timestep, which we average across time and across positive and negative examples to create the final loss.  This approach allows us to estimate the entire sequence of ratios $\{p_{\btheta}(\bx_t \mid \by_{t+1:T})/p_{\btheta}(\bx_t)\}_{t=1}^{T-1}$ using a single RNN backwards pass.

\subsection{The SIXO Bound Can Become Tight}
\label{app:sec:methods:sixo_bound:tightness}

\setcounter{thm}{0}
\begin{prop}
{\normalfont \textbf{(Reproduced from Section \ref{sec:methods:tightness})} Sharpness of the SIXO bound.} Let $p(\bx_{1:T}, \by_{1:T})$ be a latent variable model with Markovian structure as defined in Section \ref{sec:background}, let $\mathcal{Q}$ be the set of possible sequences of proposal distributions indexed by parameters $\btheta \in \Theta$, and let $\mathcal{R}$ be the set of possible sequences of positive, integrable twist functions indexed by parameters $\bpsi \in \Psi$. Assume that $\{p(\bx_t \mid \bx_{t-1}, \by_{1:T})\}_{t=1}^{T} \in \mathcal{Q}$ and $\{p(\by_{t+1:T} \mid \bx_t)\}_{t=1}^{T-1} \in \mathcal{R}$. Finally, assume $\mathcal{L}_\mathrm{SIXO}(\btheta, \bpsi, \by_{1:T})$ has the unique optimizer $\btheta^*, \bpsi^* = \argmax_{\btheta \in \Theta, \bpsi \in \Psi} \mathcal{L}_\mathrm{SIXO}(\btheta, \bpsi, \by_{1:T})$.

Then the following holds:
\begin{enumerate}
    \item $q_{\btheta^*}(\bx_t \mid \bx_{1:t-1}, \by_{1:T}) = p(\bx_t \mid \bx_{1:t-1}, \by_{1:T})$ \ for \ $t=1,\ldots, T$,
    \item $r_{\bpsi^*}(\by_{t+1:T}, \bx_t) \propto p(\by_{t+1:T} \mid \bx_t)$ up to a constant independent of $\bx_t$ \ for \ $t=1,\ldots, T-1$, 
    \item $\mathcal{L}_\mathrm{SIXO}(\btheta^*, \bpsi^*, \by_{1:T}) = \log p(\by_{1:T})$ for any number of particles $K \geq 1$.
\end{enumerate} 
\end{prop}
\begin{proof}
First we reproduce the proof in the main text that $\mathcal{L}_\mathrm{SIXO}(\btheta, \bpsi, \by_{1:T}) \leq p(\by_{1:T})$ for any setting of $\btheta, \bpsi$. As previously,  fix $r_{\bpsi}(x_T) = 1$ and let $\widehat{Z}_\mathrm{SIXO}(\btheta, \bpsi, \by_{1:T})$ be the normalizing constant estimator returned by SMC run with targets $\{p(\bx_{1:t}, \by_{1:t})r_{\bpsi}(\by_{t+1:T},\bx_t)\}_{t=1}^T$ and proposals $\{q_{\btheta}(\bx_t \mid \bx_{1:t-1}, \by_{1:T}))\}_{t=1}^T$. Then,
\begin{align}
  \mathcal{L}_\mathrm{SIXO}(\btheta, \bpsi, \by_{1:T}) &\triangleq \mathbb{E}\left[ \log \widehat{Z}_\mathrm{SIXO}(\btheta, \bpsi, \by_{1:T})\right]\\
  &\leq \log \mathbb{E}\left[\widehat{Z}_\mathrm{SIXO}(\btheta, \bpsi, \by_{1:T})\right]\label{eq:app:jensens-bound}\\
  &= \log p(\by_{1:T}), \label{eq:app:marginal-bound}
\end{align}
where \eqref{eq:app:jensens-bound} holds by Jensen's inequality and the concavity of $\log$, and \eqref{eq:app:marginal-bound} holds by the unbiasedness of SMC's marginal likelihood estimator.

Because $\mathcal{L}_\mathrm{SIXO}(\btheta, \bpsi, \by_{1:T}) \leq p(\by_{1:T})$ and we assume $\mathcal{L}_\mathrm{SIXO}(\btheta, \bpsi, \by_{1:T})$ has a unique optimizer, any setting of $\btheta$ and $\bpsi$ that causes the bound to hold with equality must be $\btheta^*, \bpsi^* = \argmax_{\btheta, \bpsi} \mathcal{L}_\mathrm{SIXO}(\btheta, \bpsi, \by_{1:T})$. Thus, we conclude the proof by showing that the twists $\{p(\by_{t+1:T} \mid \bx_t)\}_{t=1}^{T-1}$ and proposals $\{ p(\bx_t \mid \bx_{1:t-1}, \by_{1:T})\}_{t=1}^T$ cause the bound to hold with equality.

We proceed by induction on $t$, the timestep in the SMC sweep. We will show that for $t=1,\ldots,T$, $\widehat{Z}_t = p(\by_{1:T})$. In proving this we will also show that for each $t$, $w_t^{k}$ either equals $1$ or $p(\by_{1:T})$ for $k=1,\ldots,K$, depending on whether resampling occurred.

For $t=1$ note that
\begin{enumerate}
    \item $\gamma_1(\bx_1^k) = p(\bx_1^k, \by_1)p(\by_{2:T} \mid \bx_1^k)$,
    \item $\gamma_0 \triangleq 1$,
    \item and $q_1(\bx_1^k) = p(\bx_1 \mid \by_{1:T})$.
\end{enumerate}
Taken together this implies that the incremental weight $\alpha_1^k$ is
\begin{align}
\alpha_1^k = \frac{p(\bx_1^k, \by_1)p(\by_{2:T} \mid \bx_1^k)}{p(\bx_1^k \mid \by_{1:T})} = \frac{p(\bx_1^k, \by_{1:T})}{p(\bx_1^k \mid \by_{1:T})} = p(\by_{1:T})
\end{align}
which does not depend on $k$. Because $w_0^k \triangleq 1$, we have that $w_1^k = w_0^k\alpha_1^k = p(\by_{1:T})$ for all $k$. This in turn implies
\begin{align}
    \widehat{Z_1/Z_0} = \frac{\sum_{k=1}^K w_1^k}{\sum_{k=1}^K w_0^k} = \frac{Kp(\by_{1:T})}{K} = p(\by_{1:T}),
\end{align}
which when combined with the fact that $\widehat{Z}_0 \triangleq 1$ yields
\begin{align}
    \widehat{Z}_1 = \widehat{Z}_0(\widehat{Z_1/Z_0}) = p(\by_{1:T}).
\end{align}

If resampling occurs at the end of step 1, all weights $w_1^{1:K}$ will be set to $1$. Thus we have shown that $\widehat{Z}_1 = p(\by_{1:T})$ and $w_1^{1:K} = p(\by_{1:T})$ or $1$.

Now assume that $\widehat{Z}_{t-1} = p(\by_{1:T})$ and $w_{t-1}^{1:K}$ equals $1$ or $p(\by_{1:T})$. Again, we derive the incremental weights $\alpha_t^k$ by noting that
\begin{enumerate}
    \item $\gamma_t(\bx_{1:t}^k) = p(\bx_{1:t}^k, \by_{1:t})p(\by_{t+1:T} \mid \bx_t^k)$,
    \item $\gamma_{t-1}(\bx_{1:t-1}^k) = p(\bx_{1:t-1}^k, \by_{1:t-1})p(\by_{t:T} \mid \bx_{t-1}^k)$,
    \item and $q_t(\bx_t^k) = p(\bx_{t}^k \mid \bx_{1:t-1}^k, \by_{1:T})$
\end{enumerate}
which yields $\alpha_t^k$ as
\begin{align}
    \alpha_t^k &= \frac{p(\bx_{1:t}^k, \by_{1:t})p(\by_{t+1:T} \mid \bx_t^k)}{p(\bx_{1:t-1}^k, \by_{1:t-1})p(\by_{t:T} \mid \bx_{t-1}^k)p(\bx_{t}^k \mid \bx_{1:t-1}^k, \by_{1:T})}\\
    &= \frac{p(\bx_{1:t}^k, \by_{1:T})}{p(\bx_{1:t-1}^k, \by_{1:T})p(\bx_{t}^k \mid \bx_{1:t-1}^k, \by_{1:T})}\\
    &= \frac{p(\bx_{1:t-1}^k, \by_{1:T})p(\bx_{t}^k \mid \bx_{1:t-1}, \by_{1:T})}{p(\bx_{1:t-1}^k, \by_{1:T})p(\bx_{t}^k \mid \bx_{1:t-1}^k, \by_{1:T})}\\
    &= 1
\end{align}
for $k=1,\ldots,K$.

Now there are two cases depending on the value of the weights at the previous timestep. If $w_{t-1}^{1:K} = 1$, then $w_t^k = w_{t-1}^kw_t^k = 1$ for all $k$, implying that $\widehat{Z_t/Z_{t-1}} = 1$. Alternatively, if $w_{t-1}^{1:K} = p(\by_{1:T})$ then $w_t^k = p(\by_{1:T})$ for all $k$ which also implies that $\widehat{Z_t/Z_{t-1}} = 1$. Given that $\widehat{Z_t/Z_{t-1}} = 1$ in both cases, and that $\widehat{Z}_{t-1} = p(\by_{1:T})$, we have that
\begin{align}
    \widehat{Z}_t = \widehat{Z}_{t-1}(\widehat{Z_t/Z_{t-1}}) = p(\by_{1:T}).
\end{align}
Finally, if resampling occurs then the weights $w_t^{1:K}$ will be set to 1. Thus we have shown that $\widehat{Z}_t = p(\by_{1:T})$ and $w_t^{1:K} = p(\by_{1:T})$ or 1 for each $t=1,\ldots,T$.
\end{proof}

Note that we have incidentally shown that all weights are equal at each step of SMC for the optimal proposals and twisting functions. This implies that the variance of the importance weights is minimized (i.e. is 0), and if effective sample size is used to trigger resampling, resampling will never occur.

\section{Experiments}
\label{app:sec:experiments}
Code for reproduction of all experiments is released here: \url{https://github.com/lindermanlab/sixo}.  Links to Weights and Biases experimental logs are also contained in the GitHub repository.

\subsection{Gaussian Drift Diffusion}
\label{app:sec:experiments:gdd}

\subsubsection{Model Details}
The one-dimensional Gaussian drift-diffusion process has joint distribution:
\begin{align*}
    p_{\btheta} (\bx_{1:T}, \by_{1:T}) \, &= \,  p_{\btheta} \left(\bx_{1:T}, y_T \right) \, = \, p_{\btheta}(x_1) \left(\prod_{t=2}^T p_{\btheta} (x_{t} \mid x_{t-1})\right) p_{\btheta}(y_T  \mid  x_T), \\
    &= \, \cn \left(x_1 \ ; \alpha, 1 \right) \left(\prod_{t=2}^T \cn \left(x_{t} \ ; \ x_{t-1} + \alpha, 1 \right) \right) \cn \left(y_T \ ; \ x_T + \alpha, 1 \right),
\end{align*}
where the free parameters of the model are $\btheta = \left\lbrace \alpha \right\rbrace \in \BTheta = \mathbb{R}$, the state is $\bx_t \in \cx = \mathbb{R}$, and the observed data are $\by_{1:T} = y_T \in \mathbb{R}$.  Training data are sampled from this joint distribution with $\alpha=1$. Note that the distributions we show in Figure \ref{fig:banner} were generated with $\alpha = 0$.

\subsubsection{Analytic Forms} 
\label{analytic_forms_sec}
The $t$\textsuperscript{th} marginal of the filtering distribution for $t < T$ is
\begin{align}
    p_{\btheta} (x_t) \, &= \, \cn ( x_t\ ;\ t \alpha, t).
\end{align}

The $t$\textsuperscript{th} marginal of the smoothing distributions can be derived as follows:
\begin{align}
    p(x_t \mid y_T) &\propto p(x_t)p(y_T \mid x_t), \\
    &= \cn \left( x_t; t\alpha, t \right) \, \cn \left(y_T; x_t + \alpha(T-t+1), T-t+1\right), \nonumber \\
    &= \cn \left( x_t; t\alpha, t \right) \, \cn \left(x_t; y_T - \alpha(T-t+1), T-t+1\right). \label{app:equ:smooth1}
\end{align}
Noting that the product of two Gaussian densities is also Gaussian:
\begin{align*}
\cn \left( x; \mu_1, \sigma_1^2 \right) \cn \left( x; \mu_2, \sigma_2^2 \right) \propto \cn \left( x; \frac{\sigma_2^2\mu_1 + \sigma_1^2\mu_2}{\sigma_1^2 + \sigma_2^2}, \frac{\sigma_1^2\sigma_2^2}{\sigma_1^2 + \sigma_2^2} \right), 
\end{align*}
allows us to combine the two Gaussian distributions in \eqref{app:equ:smooth1}:
\begin{align*}
    p(x_t \mid y_T) &\propto \cn \left( x_t; \frac{(T-t+1)t\alpha + t(y_T-\alpha(T-t+1))}{t + (T-t+1)}t, \frac{t(T-t+1)}{t+(T-t+1)}\right)\\
    &= \cn \left( x_t; \frac{t}{T+1}y_T, \frac{t(T-t+1)}{T+1}\right).
\end{align*}
Hence the smoothing distribution is a Gaussian distribution with analytically computable mean and variance terms. 

The filtering and smoothing distributions are equal at $t=T$.  It is interesting to note that for $\alpha = \frac{y_T}{T+1}$, which is the maximum likelihood drift parameter for a \emph{single} datapoint $y_T$, the sequence of filtering and smoothing distributions have the same means.  However, the variances are different for all $t < T$; in particular, the smoothing distribution variance peaks in the middle of the timeseries, whereas the variance of the filtering distribution is monotonically increasing for $t<T$, and then drops at $t=T$.

According to Proposition \ref{prop:sixo}, we expect the proposal recovered by SIXO, $q_{\theta_t}$, to match the conditional of the smoothing distribution:
\begin{align*}
    q_{\btheta_1} (x_1 \mid y_T) &= p_{\btheta} (x_1 \mid y_T) \, \quad \quad \, \, = \, \cn \left( x_1\ ;\ \frac{y_T}{T+1}, \frac{T}{T+1}\right) \quad \text{ for } t = 1, \\
    q_{\btheta_t} (x_t \mid x_{t-1}, y_T) &= p_{\btheta} (x_t \mid x_{t-1}, y_T) \, = \, \cn \left( x_t\ ;\ \frac{(T-t+1)x_{t-1} + y_T}{T-t+2}, \frac{T-t+1}{T-t+2}\right) \ \text{ otherwise }.
\end{align*}
Note that the mean is an affine function with bias equal to zero.  

Furthermore, we also expect the optimal twist distribution to be equal to the true lookahead distribution:
\begin{align}
    r_{\bpsi_t}(y_T \mid x_t) \, = \, p_{\btheta} (y_T \mid x_t) \, = \, \cn \left( y_T\ ;\ x_t + \alpha(T-t+1), T-t+1\right) \quad \forall t \in 1, \ldots, T-1. \label{app:equ:gdd:tilt}
\end{align}

\subsubsection{SIXO Variants for the Gaussian Drift Diffusion}

\begin{figure}[t]
    \centering
    \includegraphics{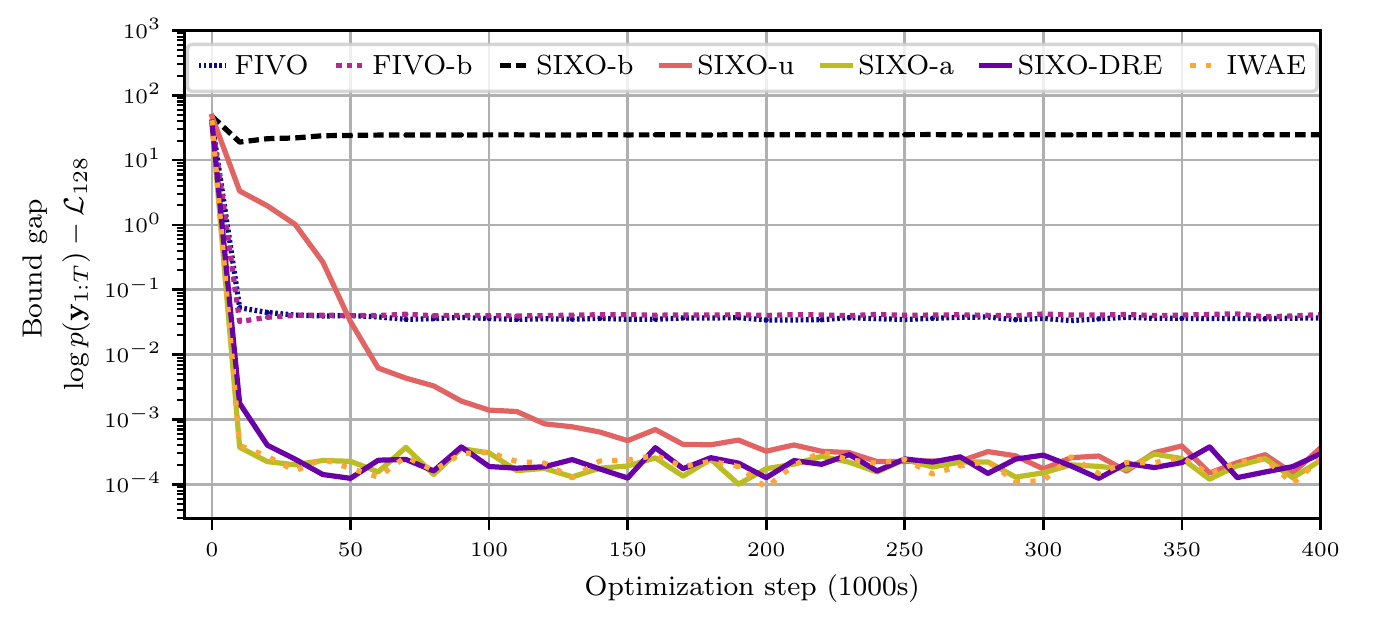}
    \vspace*{-0.5cm}
    \caption{Convergence of the bound for all methods discussed in Section \ref{app:sec:experiments:gdd}. Of these lines, FIVO-b, SIXO-b and SIXO-a were omitted from Figure \ref{fig:exp:gdd:results} the main text.  }
    \label{app:fig:full_bounds}
\end{figure}

For all experiments we use a Gaussian proposal at each timestep, parameterized as $q_{\btheta_t}(x_t \mid x_{t-1}, y_T) = \mathcal{N}(x_t; f_t(x_{t-1}, y_T), \sigma_{qt}^2)$ where $f_t$ is a general affine function of $x_{t-1}$ and $y_T$.  There are therefore $4T-1$ proposal parameters to learn ($T$ biases, $T$ $y_T$ coefficients, $T-1$ $x_{t-1}$ coefficients, and $T$ variances $\sigma_{qt}^2$).

We test four variants of SIXO:
\begin{enumerate}
    \item \textbf{SIXO-u} learns $\btheta$ and $\bpsi$ by gradient ascent on the unified bound given in \eqref{equ:sixo_bound} using the unbiased gradients (reparameterization and score-function gradients). We parameterize the twists as $r_t(y_T, x_t) = \mathcal{N}(y_T; g_t(x_t), \sigma_{rt}^2)$ for $t < T$, where $g_t$ is a learnable affine function and $\sigma_{rt}^2$ is also learned.
    \item \textbf{SIXO-DRE} as defined in Algorithm \ref{algo:sixo-dre} learns $\btheta$ by ascending the bound \eqref{equ:sixo_bound} \emph{using the biased reparameterization gradients} as in \eqref{eq:app:biased-grad}. The twist parameters $\bpsi$ are then fit using a density ratio update. The twist is parameterized as an MLP that produces the coefficients of a quadratic function over $\bx_t$ as a function of $y_T$ and $t$.  This quadratic function is evaluated at $\bx_t$ to compute the $\log r_{\psi}$ value.  
    \item \textbf{SIXO-a} (not included in Figure \ref{fig:exp:gdd:results}) uses the analytic form for the twist as a function of $\btheta$ and $y_T$ (specified in \eqref{app:equ:gdd:tilt}).  There are no free parameters to learn for this twist, and $\btheta$ is learned by ascending the bound \eqref{equ:sixo_bound} using the biased reparameterization gradients in \eqref{eq:app:biased-grad}.  
    \item \textbf{SIXO-b} (not included in Figure \ref{fig:exp:gdd:results}) learns $\btheta$ and $\bpsi$ by gradient ascent on the unified bound given in \eqref{equ:sixo_bound} using biased reparameterization gradients \eqref{eq:app:biased-grad}. We parameterize the twists as $r_t(y_T, x_t) = \mathcal{N}(y_T; g_t(x_t), \sigma_{rt}^2)$ for $t < T$, where $g_t$ is a learnable affine function and $\sigma_{rt}^2$ is also learned.
\end{enumerate}
Note that the true distributions lie within the variational families (assuming a sufficiently expressive MLP for SIXO-DRE, which is not unreasonable).  In all of these models we initialize the parameter $\alpha = 0$.  

\begin{figure*}[h!]
    \centering
    
    \begin{subfigure}[t]{0.5\textwidth}
        \includegraphics[width=\textwidth]{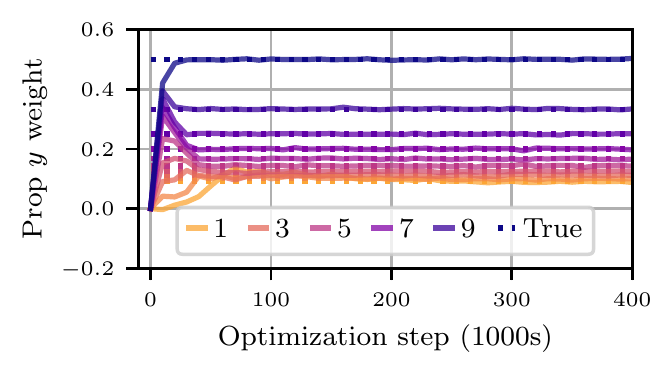}
        \vspace*{-0.5cm}
    \end{subfigure}%
    \hfill%
    \begin{subfigure}[t]{0.5\textwidth}
        \includegraphics[width=\textwidth]{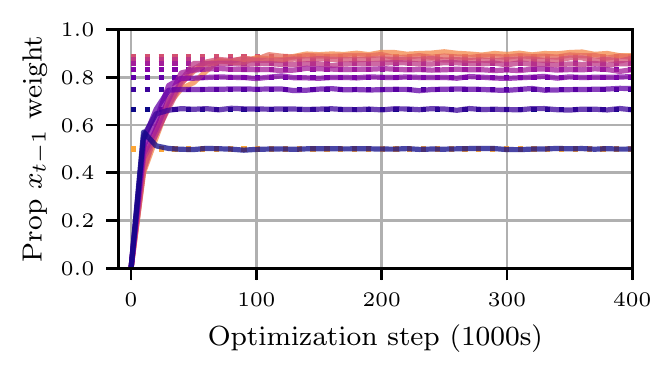}
        \vspace*{-0.5cm}
    \end{subfigure}%
    
    \begin{subfigure}[t]{0.5\textwidth}
        \includegraphics[width=\textwidth]{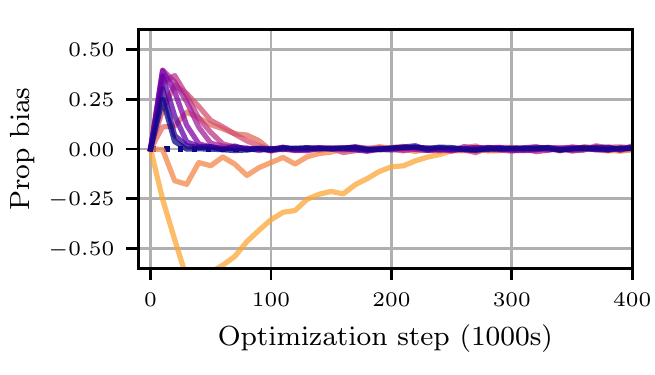}
        \vspace*{-0.5cm}
    \end{subfigure}%
    \hfill%
    \begin{subfigure}[t]{0.5\textwidth}
        \includegraphics[width=\textwidth]{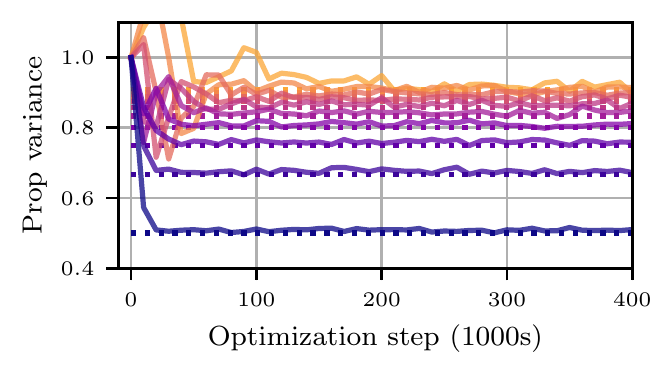}
        \vspace*{-0.5cm}
    \end{subfigure}%
    
    \begin{subfigure}[t]{0.5\textwidth}
        \includegraphics[width=\textwidth]{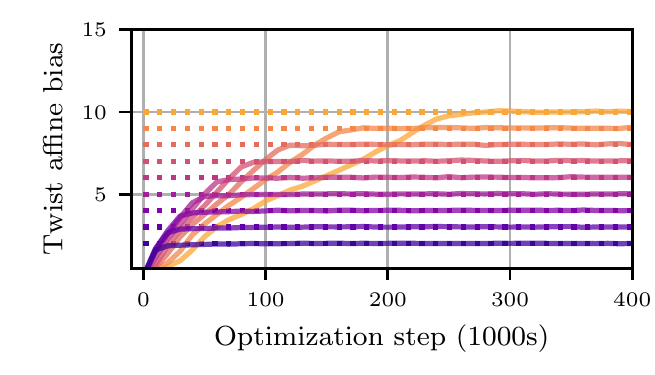}
        \vspace*{-0.5cm}
    \end{subfigure}%
    \hfill%
    \begin{subfigure}[t]{0.5\textwidth}
        \includegraphics[width=\textwidth]{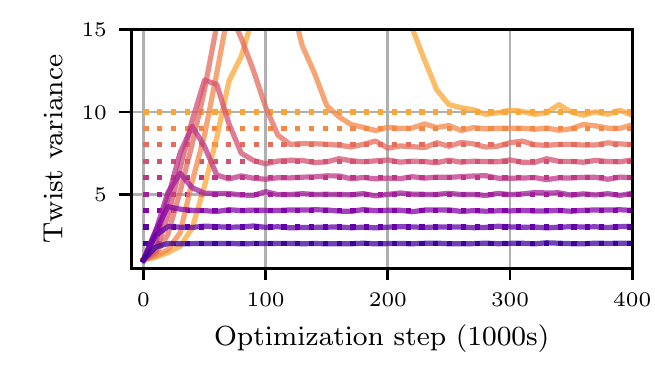}
        \vspace*{-0.5cm}
    \end{subfigure}%
    
    \begin{subfigure}[t]{0.5\textwidth}
        \includegraphics[width=\textwidth]{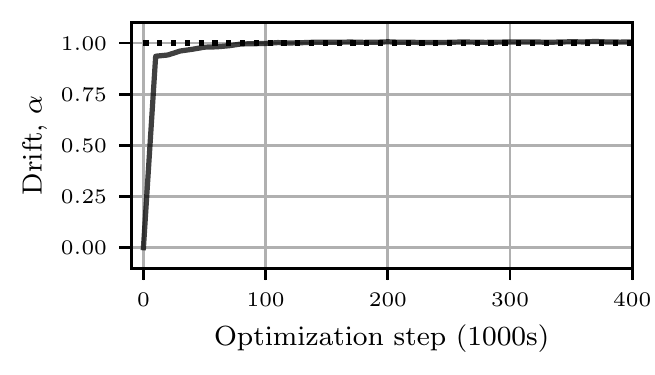}
        \vspace*{-0.5cm}
    \end{subfigure}
    
    \caption{Convergence for all free parameters in the GDD experiment using SIXO-u.  The drift parameter ($\alpha$) is constant across time.  Coloring of lines indicates the time $t \in 1, \ldots, 10$.}
    \label{fig:exp:gdd_full}
\end{figure*}

\subsubsection{Results}

We compare the four different variants of SIXO to IWAE, FIVO with biased gradients, and FIVO with unbiased gradients.  For clarity, we omitted some of these comparisons from Figure \ref{fig:exp:gdd:results} in the main text, but include them here in Figure \ref{app:fig:full_bounds}.  Individual seeds for each experiment were run using two CPU cores, 8Gb of memory, and had a runtime of no longer than five hours.

The model drift $\alpha$ is initialized to zero, affine function weights are initialized to zero, affine function biases are initialized to zero, and $\sigma_{qt}^2$ and $\sigma_{rt}^2$ are initialized to one.

In Figure \ref{app:fig:full_bounds} we show the convergence of the bound across all seven methods we considered. FIVO with unbiased gradients (FIVO) performs comparably to FIVO with biased gradients (FIVO-b), however, converges slightly more slowly.  We find that SIXO-b performed worse than all the other methods, similar to the results in \cite{lawson2018twisted}.  We therefore omitted it from the main paper for brevity.  However, this result motivated us to find alternative methods.  

All of SIXO-u, SIXO-DRE and SIXO-a converge to the correct solution and achieve a tight variational bound (parameter convergence is shown in Figure \ref{fig:exp:gdd_full} for SIXO-u).  SIXO-u converges to a tight bound most slowly, but SIXO-DRE and SIXO-a converge quickly at a rate similar to IWAE.  This result shows that SIXO is able to recover the optimal model, true posterior, and true twists, and that it can recover a tight variational bound.  It also shows that SIXO-DRE converges in a rate commensurate with the best-case convergence of SIXO-a.  Crucially, it shows that by using SIXO-DRE we can circumvent the need for the high-variance score-function gradient terms, and still recover the correct solution.  Future work will investigate whether the use of a twist makes the reparameterization gradient less biased.

Figures \ref{fig:banner:lineages:fivo} and \ref{fig:banner:lineages:sixo} show the lineages for two SMC sweeps: one using a model and proposal learned using FIVO, and one using a model, proposal and twist learned using SIXO-u.  We resample at every timestep using systematic resampling.  In each color we show the unweighted filtering particles prior to resampling (i.e. the propagated particles) at time $t$, and then the smoothing particles after performing a backwards pass from that timestep.  As such, we are showing the family of smoothing and predictive distributions.  We see that FIVO proposes particles towards the observation, but that these particles are often preferentially resampled back towards the prior distribution.  This results in gross particle degeneracy.  In contrast, the lineages for SIXO are nearly perfect, with every particle being resampled exactly once.  

\subsection{Stochastic Volatility Model (SVM)}
\label{app:sec:experiments:svm}

\subsubsection{Model Details}
The SVM models an $N$-dimensional state-space is defined as follows:
\begin{align}
    p_{\btheta}\left(\bx_{1:T}, \by_{1:T} \right) &= p_{\btheta} \left( \bx_1 \right) p_{\btheta} \left( \by_1  \mid  \bx_1 \right) \prod_{t=2}^T p_{\btheta}( \bx_t \mid \bx_{t-1}) p_{\btheta}( \by_t \mid \bx_t), \\  
    \bx_1 \sim \mathcal{N} \left( \mathbf{0}, \mathbf{Q} \right) & , \quad 
    \bx_t = \bmu + \bphi \left( \bx_{t-1} - \bmu \right) + \boldsymbol\nu_t , \quad \by_t = \boldsymbol\beta \ \mathrm{exp} \left( \frac{\bx_t}{2} \right) \mathbf{e}_t , 
\end{align}
where the states and observations are defined as:
\begin{align}
    \bx_{1:T} = \bx_{1:T} \in \mathcal{X}^T = \mathbb{R}^{T \times N} , \quad \by_{1:T} = \by_{1:T} \in \mathcal{Y}^T = \mathbb{R}^{T \times N} , 
\end{align}
the transition and observation noise terms are defined as:
\begin{align}
    \boldsymbol\nu_t \sim \mathcal{N} \left( \mathbf{0}^{N}, \mathbf{Q} \right) , \quad \mathbf{e}_t \sim \mathcal{N} \left( \mathbf{0}^{N}, \mathbf{1}^{N} \right).
\end{align}
and all multiplications are performed element-wise.  The model has free parameters defined as:
\begin{align}
    \theta = \left\lbrace \bmu, \bphi, \boldsymbol\beta , \mathbf{Q} \right\rbrace, \quad \mathrm{where}  \quad  \bmu \in \mathbb{R}^N , \quad \bphi \in \left[ 0, 1 \right]^N , \quad \boldsymbol\beta \in \mathbb{R}^N_+ , \quad \mathbf{Q} \in \mathrm{diag}(\mathbb{R}^N_+) ,
\end{align}
such that there are $4N$ model parameters.  The model learning objective is to recover the free parameters, $\btheta$, given observed data $\by$.  The data we consider are the monthly returns from $N=22$ currencies over the period from 9/2007 to 8/2017, transformed into the log domain.  As a result, $\mathcal{X} = \mathcal{Y} = \mathbb{R}^{119 \times 22}$.

\subsubsection{Results}
We use a fixed proposal per timestep parameterized as a Gaussian perturbation to $p_{\btheta}$ as in \cite{naesseth2018variational},
\begin{align}
q_{\btheta}(\bx_t  \mid  \bx_{t-1}, \by_{1:T}) \propto p_{\btheta}(\bx_t  \mid  \bx_{t-1})\mathcal{N}(\bx_t; \boldsymbol{\mu}_t, \mathbf{\Sigma}_t).
\end{align}
Thus the parameters specific to $q$ are the means $\boldsymbol{\mu}_t \in \mathbb{R}^N$ and covariance matrices $\mathbf{\Sigma}_t \in \mathrm{diag}(\mathbb{R}^N_+)$ for $t=1,\ldots,T$.

SIXO-q uses a parameterless quadrature twist, so we skip twist optimization and use biased reparameterization gradients of the SIXO bound to fit the model and proposal. For the quadrature twist we use Gauss-Hermite quadrature with degree five.

For SIXO-DRE, we model the twist using the backwards RNN method introduced in Section \ref{sec:methods:dre}.  The twist is parameterized with a one-layer RNN with 128 hidden units and a one-layer MLP with 128 hidden units.  The twist is learned using the alternating DRE method described in Section \ref{sec:methods:dre}.  We generate a batch of 32,000 synthetic trajectories from the current model, and perform minibatch stochastic gradient descent using the ADAM optimizer~\citep{kingma2014adam} with a learning rate of $3\times 10^{-3}$ and a minibatch size of 64.  We apply 1,000 twist updates (corresponding to two epochs) and then apply 1,000 updates to the model and proposal.  The model and proposal parameters are updated using the ADAM optimizer~\citep{kingma2014adam} with a learning rate of $1\times 10^{-4}$.  We use four particles per SMC sweep, and average across four datasets per model and proposal update.  

Individual FIVO and SIXO-q seeds for each experiment were run using four CPU cores, 24Gb of memory, and had a wallclock time of no longer than twenty four hours.  SIXO-DRE had a longer runtime of five days, however, competitive results were achieved within two days.  

\paragraph{Train Bound Performance}  We compare three methods:  FIVO (learning $\btheta$), SIXO-q (learning $\btheta$ with a quadrature twist) and SIXO-DRE (learning $\btheta$ and $\bpsi$ using a DRE twist).  All methods use the biased score-function gradients of \eqref{eq:app:biased-grad}.  We show the median and quartiles across five random seeds.  Each $\mu_n$ is initialized from $\mathcal{N}(0, 0.3)$ (with $n \in 1, \ldots 22$).  $\phi_n$ is learned in the unconstrained space $\mathbb{R}$, and is transformed to $[0,1]$ by passing the raw $\phi$ through a hyperbolic tangent function. The unconstrained $\phi_n$'s are initialized from $\mathcal{N}(\mathrm{arctanh}(0.1), 0.3)$, and  $\log \beta_n$ is initialized from $\mathcal{N}(\log 1.0, 0.3)$.  Finally, $\log Q_n$ is initialized from $\mathcal{N}(\log 1.0 , 0.3)$.  The train bound value we report is taken as the average train bound once converged.  We also report $\mathcal{L}_{\mathrm{BPF}}^{2048}$ as the bound evaluated using a BPF with 2,048 particles.  This tests the performance of the learned model opposed to inference performance.

A one-way ANOVA \cite{fisher1992statistical} rejected the null hypothesis that the mean train bounds are equal ($p < 0.0001$), and a post-hoc Dunnett T3 test~\citep{dunnett1980pairwise} found the mean SIXO-DRE bound to be the highest ($p < 0.01$ for all pairs). For the $\mathcal{L}_{\mathrm{BPF}}^{2048}$ values, a one-way ANOVA failed to reject the null hypothesis that the train bounds are equal ($p = 0.25 > 0.05$), so all entries are bolded.  This methodology was recommended in \citet{sauder2019updated}.

\paragraph{Test Set Performance}  We also report the performance on a held-out dataset constructed from the new data since \citet{naesseth2018variational} was published.  The test $\mathcal{L}_{\mathrm{BPF}}^{2048}$ we report is the bound evaluated using a BPF with 2,048 particles, averaged across all checkpoints after $75\%$ of training. A one-way ANOVA rejected the null hypothesis that the mean test bounds are equal ($p < 0.0001$), and a post-hoc Dunnett T3 test found the mean SIXO-DRE bound to be the highest ($p < 0.01$ for all pairs).

\subsection{Hodgkin-Huxley Model}
\label{app:sec:experiments:hh}
We provide a brief overview of the model here for completeness, but refer the reader to Chapter 5.6 of \citet{dayan2005theoretical} for more detailed information. 

\subsubsection{Model Details}
The HH model is a physiologically grounded model of neural action potentials~\citep{hodgkin1952quantitative} defined through a set of four nonlinear differential equations.  Each neuron is defined by four state variables: the instantaneous membrane potential $v(t) \in \mathbb{R}$, the potassium channel activation $n(t) \in [0, 1]$, sodium channel activation $m(t) \in [0, 1]$, and sodium channel inactivation $h(t) \in [0, 1]$.  The channel states represent the aggregated probability that the given channel is active.  The state evolves according to:
\begin{equation}
    C_m \frac{\d v(t)}{\d t} = i_{\mathrm{ext}}(t) - g_L (v(t) - E_{L}) - g_K n^4 (v(t) - E_K) - g_{Na} m^3 h (v(t) - E_{Na}). \label{equ:hh:dynamics}
\end{equation}
The membrane capacitance $C_m$ is often defined to be $1.0$.  The first term, $i_\mathrm{ext}(t)$ is the externally injected current.  The second term represents the net current through cell membrane due to the potential difference between the intracellular and extracellular mediums, often referred to as the \emph{leakage current}.  This current is a function of the membrane capacitance, $g_L$, and the potential of the extracellular medium, $E_L$, where the potential difference across the membrane is then $v(t) - E_L$.  The third term represents the net current through the membrane as a result of the potassium channels as a function of the channel state $n(t)$, the channel capacitance $g_K$, and the potassium reversal potential $E_K$.  The final term represents the current through the membrane as a result of the sodium channel states, both $m(t)$ and $h(t)$, the sodium channel conductance $g_{Na}$, and the sodium reversal potential $E_{Na}$.  The channel states evolve according to:
\begin{align}
    \frac{\d n(t)}{\d t} &= \alpha_{n}(v(t))(1-n)- \beta_{n}(v(t))n , \\
    \frac{\d m(t)}{\d t} &= \alpha_{m}(v(t))(1-n)- \beta_{m}(v(t))m , \\
    \frac{\d h(t)}{\d t} &= \alpha_{h}(v(t))(1-n)- \beta_{h}(v(t))h,
\end{align}
where $\alpha_{n}, \alpha_{m}, \alpha_{h}, \beta_{n}, \beta_{m}, \beta_{h}$ are all fixed scalar functions of the membrane potential. We discretize this continuous-time differential equation into a discrete-time latent variable model by integrating using Euler integration with an integration timestep of $0.02$ms (similarly to \citet{huys2009smoothing}). This defines the deterministic time-evolution of the neural state.

We follow a similar approach as \citet{huys2009smoothing} and add Gaussian random noise to each of the four states at each timestep. The membrane potential is unconstrained, and so we can add noise directly. The gate states, however, are constrained to the range $[0, 1]$.  \citet{huys2009smoothing} use truncated Gaussian noise to avoid pushing the state outside the constrained range. We use a different approach and transform the constrained states into an unconstrained state by applying the inverse sigmoid function to the raw gate value. This has the effect of modifying the variance of the perturbation in constrained space as a function of the state (heteroscedastic noise in constrained space). However, this hetereoscedasticity carries with it a favorable intuition. The magnitude of the noise term is reduced (after being pushed through a sigmoid) close the limits. This means that the same noise kernel provides smaller perturbations close to the extremes, while still retaining a larger permissible perturbations in the mid-range. We add zero-mean Gaussian noise to the potential with variance scaled by the integration timestep, $\sigma_v^2 = 9 \mathrm{mVs^{-1}} \times 0.02\mathrm{s} = 0.18\mathrm{mV}$. The unconstrained gate variables are perturbed by zero-mean Gaussian noise with variance also scaled by the integration timestep, $\sigma_{\{n,m,h\}}^2 = 0.1\mathrm{s^{-1}} \times 0.02\mathrm{s} = 0.002$. Observations are sampled from a Gaussian emission distribution centered on the current membrane potential with variance 25mV. Observations are generated every 50 timesteps. The model is integrated with a timestep of 0.02ms, corresponding to an acquisition rate of 1kHz. 

We initialize the potential according to a Gaussian distribution with mean equal to -65mV and a standard deviation of 25mV. The unconstrained gate variables are initialized from a Gaussian distribution with mean defined by an estimate of the steady-state value $x = \mathrm{sigmoid} \left(\nicefrac{\alpha_x (-65)}{ \alpha_x (-65) - \beta_x (-65) }\right)$, where $x$ represents the $n$, $m$ or $h$ states.

To iterate the model, we first constrain the state by passing the gate variables through a sigmoid. The potential is already unconstrained and so requires no transform.  We then iterate the model given the constrained state. The iterated state is then unconstrained by passing the gate states through the logit function (inverse of the sigmoid function). The noise term is then added to the iterated, unconstrained state. Observations are then generated by sampling from the emission distribution every 50 steps. We generate traces with 2,048 timesteps, corresponding to approximately 40ms.

\subsubsection{Results}
For the experiments presented in Section \ref{sec:exp:hh} we use a bootstrap proposal, i.e. $q_{\btheta}(\bx_t \mid \bx_{t-1}, \by_{1:T}) = p_{\btheta}(\bx_t \mid \bx_{t-1})$. This was to focus on learning the twist and show that even learning just a twist can markedly improve inference and model learning outcomes. Individual seeds for each experiment were run using eight CPU cores, 16Gb of memory, and had a runtime time of no longer than two days. 

\paragraph{Inference} In Figure \ref{fig:exp:hh:twisted_smc} we compare the inference performance of a BPF to that of SIXO with a DRE twist. Both models (and hence proposals) are identical. We use $128$ particles in the forward (twisted) SMC sweep. To parameterize the twist we use a recurrent neural network (RNN) with 32 hidden units and a one-layer MLP head with 32 hidden units. This is learned using the DRE update described in Section \ref{sec:methods:dre}. We use minibatched stochastic gradient descent with the ADAM optimizer~\citep{kingma2014adam}, a learning rate of $0.01$, and a minibatch size of 32. We sample 2,048 length $T$ synthetic trajectories from the current model, from minibatches are constructed. To handle missing observations, we apply the RNN to just the observations, and then linearly interpolate the output of the RNN between observations on input to the head MLP. While the twist encoding input to the head MLP is dependent on the previous observation, this approach is still permitted under the twisted SMC framework, and dramatically reduces the cost of application by only processing valid observations, and, allows the twist to create meaningful encodings between observations without introducing more parameters. Investigating alternative twisting architectures for handling missing observations is an interesting avenue of future research.

\paragraph{Model Learning}
For the experiments presented in Figure \ref{fig:hh_results} and Table \ref{tab:hh} in Section \ref{sec:exp:hh} we learn just the constant external current input, $i_{\mathrm{ext}}$ in \eqref{equ:hh:dynamics}. We generate 10,000 training sequences, and evaluate on 30 sequences, both generated with $i_{\mathrm{ext}} = 13\mathrm{mV}$. We show median across 15 random seeds, with shaded regions showing the $10^{th}$ to $90^{th}$ deciles. Across the 15 seeds, we deterministically initialize $i_{\mathrm{ext}}$ to a value uniformly spaced between $1.3\mu \mathrm{A}$ and $37.7\mu \mathrm{A}$, corresponding to a relative error of between $-0.9$ and $1.9$. We report the train bound value as the average train loss once the loss had converged. We evaluate the test bound loss $\mathcal{L}_{\mathrm{BPF}}^{256}$ using a bootstrap particle filter on the held-out validation sequences using 256 particles.

We use $K=4$ particles per SMC sweep when evaluating the model and proposal gradients (c.f. Line 11 of Algorithm \ref{algo:sixo-dre}), and average across four sequences per parameter update. We found that the gradients were prone to shot noise, causing huge jumps in the parameter value and optimizer state. We therefore investigated using gradient clipping. We tested a range of values and found that clipping the model gradient magnitude to 50 eradicated these jumps in SIXO. We also tested clipping with FIVO, and found that while clipping removed the jumps and reduced the variance of the parameters recovered, it increased the bias relative to not using clipping (see Figure \ref{fig:hh_results:theta}). We report the performance of FIVO with and without clipping for fairness. We use the same twist architecture described above. We take 400 steps in the twist (equating to approximately six epochs) per 100 steps in the model.

\end{document}